\documentclass[a4paper,fleqn]{cas-sc}

\usepackage{float}
\usepackage[authoryear]{natbib}
\usepackage{amsthm,amsmath,amsfonts,amssymb}
\usepackage{tabularx}
\usepackage{graphicx}
\usepackage{blindtext}
\usepackage{algorithm}
\usepackage{algpseudocode}
\usepackage{longtable}
\usepackage{booktabs} 
\usepackage{subcaption}
\usepackage{multirow}
\usepackage{mathrsfs}
\usepackage{tikz}
\usepackage{makecell}
\usepackage{booktabs}
\usepackage{longtable}
\usepackage{multirow}
\usepackage{siunitx} %
\usepackage{float}
\usepackage{caption, subcaption}

\usetikzlibrary{shapes, positioning, calc, shapes.geometric, arrows.meta, arrows, fit, decorations.markings}

\theoremstyle{plain}

\newtheorem{theorem}{Theorem}
\newtheorem{lemma}{Lemma}

\theoremstyle{definition}
\newtheorem{definition}{Definition}
\newtheorem{example}{Example}

\newtheorem{remark}{Remark}
\newtheorem{proposition}{Proposition}

\makeatletter
\newenvironment{breakablealgorithm}
{%
	\begin{center}
		\refstepcounter{algorithm}%
		\hrule height.8pt depth0pt \kern2pt%
		\renewcommand{\caption}[2][\relax]{%
			{\raggedright\textbf{Algorithm~\thealgorithm} ##2\par}%
			\ifx\relax##1\relax
			\addcontentsline{loa}{algorithm}{\protect\numberline{\thealgorithm}##2}%
			\else
			\addcontentsline{loa}{algorithm}{\protect\numberline{\thealgorithm}##1}%
			\fi
			\kern2pt\hrule\kern2pt
		}
	}{%
		\kern2pt\hrule\relax
	\end{center}
}
\makeatother

\tikzstyle{arrowstyle}=[draw]
\tikzstyle{directed}=[postaction={decorate,decoration={markings,
		mark=at position 0 with {\node[circle,fill=black,inner sep=0.3ex] {};},
		mark=at position 0.5 with {\arrow[scale=1.2,arrowstyle]{stealth};},
		mark=at position 1 with {\node[circle,fill=black,inner sep=0.3ex] {};},
}}]

\def\P{\mathcal{P}}

\newcommand{\indep}{\rotatebox[origin=c]{90}{$\models$}}
\DeclareRobustCommand\model{\mathrel{|}\joinrel\mkern-.5mu\mathrel{\neq}}
\newcommand{\ndep}{\rotatebox[origin=c]{90}{$\model$}}

\def\tsc#1{\csdef{#1}{\textsc{\lowercase{#1}}\xspace}}
\tsc{WGM}
\tsc{QE}

\begin{document}
\let\WriteBookmarks\relax
\def\floatpagepagefraction{1}
\def\textpagefraction{.001}

\shorttitle{}    

\shortauthors{Deng et.al.}  

\title [mode = title]{Estimate Collapsibility of Causal Effects in Completed Partial DAGs via Strong d-Convex Hulls} 



%

\author[1]{Yuxin Deng}

\fntext[1]{These authors contributed equally to this work.}

\ead{dorlma_yuxin@126.com}


\credit{Methodology, Validation, Writing}

\affiliation[1]{organization={College of Mathematics and System Science,
		Xinjiang University },
	city={Urumuqi},
	postcode={830000}, 
	state={Xinjiang},
	country={China}}
	\fnmark[1]
\author[2]{Yi Sun}
\fnmark[1]

\ead{brian@xjufe.edu.cn}


\credit{Conceptualization, Reviewing \& editing}

\affiliation[2]{organization={Institute of Statistics and Data Science, Xinjiang University of Finance and Economics },
            city={Urumuqi},
            postcode={830000}, 
            state={Xinjiang},
            country={China}}

\author[1]{Zhiming Li}
\cormark[1]

\ead{zmli@xju.edu.cn}


\credit{Supervision, Writing -- review \& editing, Funding Acquisition}
\cortext[1]{Corresponding author.}

\author[1]{Huaxiong Liu}


\ead{liuhuaxiong18@gmail.com}


\credit{Software}


\begin{abstract} 
This paper proposes a collapsible method for estimating causal effects that maintains the estimator's consistency before and after marginalization over some variables in completed partially directed acyclic graphs (CPDAGs). We first introduce the estimate collapsibility for CPDAGs and characterize the
minimal collapsible sets as strong d-convex hulls. An efficient algorithm is devised to obtain such sets in DAGs and is generalized to CPDAGs. Then, we combine the graph reduction procedure with the IDA framework. Finally, experiments and empirical analysis show the effectiveness of the collapsibility for causal estimations in CPDAGs. Code is available at \href{https://github.com/Jamyang-D/strongly-convex}{https://github.com/Jamyang-D/strongly-convex}.
\end{abstract}



\begin{keywords}
Completed partially DAGs \sep Causal effect \sep Estimate collapsibility \sep Strong d-convex hull  
\end{keywords}

\maketitle
\section{Introduction}
Causal effect estimation is one of the core problems in causal inference, quantifying the impact of treatments or interventions, and has been widely applied in  healthcare~\citep{imbens2015causal,PHAM2026}, online advertising~\citep{braun2025}, decision-making~\citep{shadi2025} and explainability research~\citep{Renero2026}. 
Randomized controlled trials (RCTs) are the gold standard for studying causal relationships, yet are subject to ethical, cost, and implementation limitations. Consequently, observational data serve as a practical alternative. To draw causal conclusions from observational data, researchers often rely on causal graphical models. A common framework is the causal Bayesian network (CBN), which encodes causal relationships using a directed acyclic graph (DAG) and conditional probability distributions. Given observational data from a DAG, causal effects can be precisely estimated via the back-door criterion, front-door criterion, and do-calculus \citep{pearl2009}.
In practice, however, causal discovery from observational data can only learn a Markov equivalence class of DAGs, represented as a completed partially directed acyclic graph (CPDAG) \citep{chickering2002learning}.

To estimate causal effects for CPDAGs,   \cite{maathuis2009} proposed the intervention calculus when the DAG is absent, known as the IDA algorithm, which estimates all possible causal effects by enumerating possible parent sets of the treatment variable. \cite{maathuis2015} generalized the back-door criterion from DAGs to CPDAGs, providing a sufficient condition for adjustment sets. For complex graphs with latent variables, a complete framework for finding valid adjustment sets in maximal ancestral graphs (MAGs) and CPDAGs is provided in { \cite{van2014,van2019}}. Building on this, { \cite{perkovic2018}} further unified these approaches and provided the necessary and sufficient generalized adjustment criterion.  Although these methods are sufficient for estimating causal effects from a CPDAG, enumerating DAGs or adjustment sets remains challenging in large graphs. To address this, many studies have focused on reducing the adjustment set or the enumeration space~{\citep{hauser2012, liu2020collapsible, guo2021minimal,  henckel2022graphical, Henckel2024, wang2025}}. 

However, most of the aforementioned methods rely on additional prior knowledge or local pruning of adjustment sets, yet redundant variables unrelated to effect estimation still remain in the graph, which should be further simplified by removing them~\citep{ZHANG2026}. That is, it is necessary to search for a minimal subset of variables that includes the treatment and outcome variables for effect estimation.The Yule-Simpson paradox shows that one cannot directly marginalize irrelevant variables in the presence of confounders~\citep{simpson1951}. To address this issue, collapsibility is an effective method that ensures the same results before and after marginalization across some variables. Up to now, collapsibility has been classified into several types,  such as parametric~\citep{wermuth1987}, estimate~\citep{xie2009}, conditional independence, and model collapsibility \citep{liu2013, xie20251}.  For example, \cite{asmussen1983} provided an equivalent condition for collapsibility in graphical log-linear models. The collapsibility properties of graphical interaction models and conditional graphical models are discussed in \cite{edwards1990, didelez2004, liu2013}. Based on this, minimal collapsible sets have attracted growing research interest. \cite{madigan1990}  introduced the selective acyclic hypergraph reduction (SAHR) algorithm to identify such sets, but it was restricted to decomposable graphical models. \cite{wang2011} addressed this problem using convex subgraphs. \cite{heng2023} proposed a path absorption method. The above methods are primarily designed for undirected graphical models. For directed graphical models, ~\cite{kim2006} first introduced the concept of estimate collapsibility and showed that a variable is estimate collapsible if and only if it is removable. Building on this, ~\cite{xie2009} proposed the concepts of c-removability and t-removability, and discussed the relationships among different types of collapsibility. More recently, \cite{deng} extended the notion of c-removability from single vertices to sets, offering a more intuitive graph-theoretic characterization. \cite{xie20252} investigated the closure property of BNs under conditioning, characterized the minimal I-map, and provided an efficient identification algorithm. 

As far as we know, existing collapsibility results mainly concentrate on distributions, whereas causal effects are defined by interventional distributions, which are inherently based on probability distributions~\citep{pearl2009}.  Thus, an intriguing question naturally arises: Can the estimation of causal effects in CPDAGs be addressed through the lens of collapsibility? If so, what specific properties would such an estimator possess? 
Inspired by these, this paper aims to investigate the estimate collapsibility of causal effects in CPDAGs.  The  contributions  and novelties of this work are summarized as follows: 

(i) We first propose the estimate collapsibility of causal effects to perform effect estimation only on a local model, thereby reducing the variable dimension for CPDAGs. More importantly, the same estimators are obtained before and after marginalization over some variables in CPDAGs.

(ii) To characterize minimal collapsible sets, we provide a novel concept of strong d-convex hulls. A corresponding polynomial-time algorithm (ISCHA) is developed to identify them. The advantage of the proposed method is that it simultaneously identifies the minimal sets that satisfy the model collapsibility condition.

(iii)   Numerical simulations and empirical analysis demonstrate that the proposed method not only reduces the scale of CPDAGs and the number of equivalent DAGs to be enumerated but also achieves accurate estimation, outperforming baseline methods.

The rest of this paper is organized as follows. In Section \ref{sec2}, we review the basic concepts of causal graphs and collapsibility.  The properties of minimal collapsible sets are characterized, and identification algorithms are provided in Section~\ref{sec3}.  In Section \ref{sec4},  we extend the above results to CPDAGs and combine the ISCHA algorithm with IDA to estimate causal effects. The modified IDA algorithm overcomes the aforementioned computational bottleneck by collapsing the original CPDAG onto a minimal collapsible set. Experimental results in Section \ref{sec5} show that our method achieves accurate probabilistic reasoning and causal effect estimation while maintaining enhanced computational efficiency as the variable dimension increases. Section \ref{sec6} provides a brief conclusion. 
\section{Preliminaries}\label{sec2}
 In this section, we review the terminology and notation of DAGs, CPDAGs, causal Bayesian networks, collapsibility, and causal effects from \cite{lauritzen1996, koller2009, pearl2009}.
\subsection{ Causal Graphs} 

Within the framework of causal inference,  let $G=(V, E)$ be { a causal DAG}, where each vertex $v \in V$ represents a random variable $X_v$ and each directed edge $u\rightarrow v$ encodes a causal relationship. 
For $(x,y)\in E$, $x$ is called a parent of $y$ and $y$ is a child of $x$. The set of parents (resp. children) of $x$ is
denoted by $pa_{G}(x)$ (resp. $ch_{G}(x)$). Based on these, we define the family, neighbors, and Markov blanket of $x$ as follows:  $fa_{G}(x) = pa_{G}(x) \cup \{x\}$, $ne_{G}(x) = pa_{G}(x) \cup ch_{G}(x)$, and $mb_{{G}}(x) = fa_G(ch_{G}(x)\cup \{x\}) \backslash \{x\}$. For a subset $A \subseteq V$,  $G_A=(A,E_A)$ denotes the subgraph induced by $A$, where $E_A = E \cap \{(x, y) : x, y \in A\}$.  $G_{\bar X}$ is a manipulated graph of $G$ by deleting
all directed edges pointing at $X$~{ \citep{spirtes2001causation}.} 

A path $l_{uv}$ between vertices $u$ and $v$ is a sequence of distinct vertices 
$(u = v_0, v_1, \ldots, v_n = v)$ such that each consecutive pair $(v_{i-1}, v_i)$ is connected by an edge, denoted by $v_{i-1} \overset{G} \sim  v_i$. 
If all edges along $l_{uv}$ are directed from $u$ to $v$, then $l_{uv}$ is a directed path or causal path. 
In this case, $u$ is an ancestor of $v$ and $v$ is a descendant of $u$. 
We use $an_G(x)$ and $de_G(x)$ to denote the sets of ancestors and descendants of a vertex $x$, respectively. 
A subset $A \subseteq V$ is called {terminal} if it has no descendants, i.e., $de_{G}(A) = \emptyset$. 	
The vertices on the path $l_{uv}$ are denoted by $V(l_{uv})$, and the set of internal vertices is defined as
$V^{o}(l_{uv}) = V(l_{uv}) \setminus \{u, v\}$. A vertex $w \in V^{o}(l_{uv})$ is called a {collider} if $l_{uv}$ contains a subpath of the form $x \rightarrow w \leftarrow y$. The set of colliders on $l_{uv}$ is denoted by $V^{c}(l_{uv})$. 
If $x$ and $y$ are non-adjacent, the triple $(x, w, y)$ forms a $V$-structure. The moralization of a DAG $G$, represented by $G^m$, is obtained by connecting $x$ and $y$ in every $V$-structure and dropping all edge directions. A path $l_{uv}$ is said to be blocked by a subset $Z \subseteq V$ if either (i) $(V^{o}(l_{uv}) \setminus V^{c}(l_{uv})) \cap Z \neq \emptyset$, or
(ii)  there exists a collider $w \in V^{c}(l_{uv})$ such that $De_{{G}}(w) \cap Z = \emptyset$, where 
$De_{{G}}(w) = \{w\} \cup de_{{G}}(w)$. 	A directed cycle is a directed path that starts and ends at the same vertex. Throughout this paper, all DAGs are assumed to contain no multiple edges.

For pairwise disjoint subsets $X, Y, Z \subseteq V$, 
if every path between $X$ and $Y$ is blocked by $Z$, 
then $X$ is {d-separated} from $Y$ by $Z$ in $G$, 
and $Z$ is a {d-separator} in $G$, denoted by $X\indep Y\mid Z\,[G]$. 
Otherwise, $X$ and $Y$ are said to be d-connected given $Z$, 
denoted by $X\ndep Y\mid Z\,[G]$. 
Under the criterion of d-separation, the graph $G$ induces an independence model $$I(G) = \{\langle X,Y|Z\rangle:\ X  \indep Y\ |\ Z [G] \text{~with pairwise disjoint subsets~} X,Y,Z\subseteq V\}.$$
For any vertex subset $A\subseteq V$,  the  projection of $I(G)$ to $A$, also called the marginal independence model on $A$, is defined as
$I(G)_A  = \{\langle X,Y|Z\rangle \in I(G): X,Y,Z\subseteq A\}.$ $I(G_A)$ is the marginal model induced by $G_A$.  It directly follows that $I(G)_A\subseteq I(G_A)$, but the reverse is not guaranteed. 

Two distinct DAGs $G_1$ and $G_2$ are said to be Markov equivalent if they induce the same independence model, 
that is, ${I}(G_1)={I}(G_2)$. The collection of all DAGs Markov equivalent to a given DAG $G$ is called the Markov equivalence class of $G$, represented by a CPDAG $\mathcal{G}$.

\subsection{Causal Bayesian Networks and Collapsibility}
A causal Bayesian network (CBN) is defined as a pair $\mathcal{B}=(G,\mathcal{P}(G))$, where $G$ is { a causal DAG} and $\mathcal{P}(G)$ denotes a family of probability distributions over $X_v (v\in V)$ that factorize according to $G$, i.e., 
\begin{align}\label{formula}
	\mathcal{P}(G) = \left\{ P(x_V; \theta) = \prod_{v\in V} P(x_v \mid \operatorname{pa}_G(x_v); \theta_v) \;\right\},
\end{align}
where $\theta_v$ are unknown parameters, $\theta = \{\theta_v\}_{v\in V}$.  The parameters of $\mathcal{P}(G)$ are said to be variation independent with respect to $G$ if, for every $P \in \mathcal{P}(G)$, the parameter space of $\theta$ is a cross product $\Theta = \otimes \Theta_v$ with $\theta_v \in \Theta_v$. Consequently, the parameters of $\mathcal{B}$ are variation independent relative to the Markov equivalence class $\mathcal{G}$ of $G$ if the parameters of $\mathcal{P}(G)$ are variation independent relative to every DAG ${G}' \in \mathcal{G}$.
Further, $\mathcal{P}(G)$ is said to be faithful to $G$ if there exists a distribution $Q \in \mathcal{P}(G)$ such that $I(G) = I(Q)$. 

The causal effect of $X$ on $Y$ relies on the post-intervention distribution $P(Y \mid do(X=x))$, which generally differs from the observational distribution $P(Y \mid X=x)$. For a binary treatment $X \in \{0,1\}$, the average causal effect (ACE) is:
$$ACE = E(Y \mid do(X=1)) - E(Y \mid do(X=0)).$$
As shown by \cite{pearl2009}, if a covariate set $Z$ satisfies the back-door criterion, the interventional distribution is identifiable from observational data via covariate adjustment $$P(Y \mid do(X=x)) = \sum_{z} P(Y \mid X=x, Z=z)P(Z=z).$$

\begin{definition}[\cite{pearl2009}]
In a DAG $G$, a set of variables $Z$ satisfies the back-door criterion relative to $X, Y\notin Z$ if: (i) no element of $Z$ is a descendant of $X$; and (ii) $Z$ blocks every path from $X$ to $Y$ that contains an arrow pointing into $X$ (i.e., every back-door path).
\end{definition}

Analogous to the independence model of $G$, the distribution family $\mathcal{P}(G)$ can also be projected onto a subset $A \subseteq V$. The marginal distribution family is defined as:
$$
\mathcal{P}(G)_A = \left\{ P_A(x_A) : P_A(x_A) = \int_{\mathcal{X}_{V \setminus A}} dP(x_V),\ P(x_V) \in \mathcal{P}(G) \right\},
$$
where $\mathcal{X}_V$ is the state space of $X_V$. Let $\mathcal{P}(G_A)$ be the marginal model induced by $G_A$. Generally, $\mathcal{P}(G_A) \subseteq \mathcal{P}(G)_A$ since $I(G)_A\subseteq I(G_A)$. In what follows, we introduce a class of paths to characterize collapsibility in CBNs. 

\begin{definition}[\cite{verma1993}]\label{Inducing path}
	For a  DAG \({G} = (V, E)\) and \(R \subseteq V\), a path \(l_{xy}\) between non-adjacent vertices \(x, y \in R\) is an inducing path of \(R\) if $R \cap V^o(l_{xy}) \subseteq V^c(l_{xy}) \subseteq an_{{G}}(\{x, y\}).$
\end{definition}

From a statistical perspective, an inducing path in a DAG generalizes a chordless (or minimal) path in an undirected graph. This concept has been extended to ancestral graphs by \cite{van2019}. Let $l_{uv}$ be an inducing path of $R$ in $G$. The {inducing structure} associated with $l_{uv}$ is denoted as $ \mathrm{IS}_G(l_{uv}) = l_{uv} \cup \mathcal{L}_G(l_{uv})$, where $\mathcal{L}_G(l_{uv})$ is the set of all directed paths in $G$ from vertices in $V^c(l_{uv})$ to  $An_{G}(\{u, v\})$, with $An_{G}(\{u, v\}) = an_{G}(\{u, v\}) \cup \{u, v\}$. Notably, $ \mathrm{IS}_G(l_{uv})$ is a subgraph of $G$, though not necessarily an induced subgraph.

\begin{definition}[\cite{deng}]\label{def:mf-pair}
	Let $G = (V, E)$ be a DAG and $R = V \setminus M$ for $M\subseteq V$. For any two non-adjacent vertices $x, y\in R$, if $x \ndep y \mid S[G]$ for all subsets $S \subseteq R \setminus \{x, y\}$, $\{x,y\}$ is called a marginalization forbidden pair (mf-pair). The set of all such mf-pairs of $R$ in $G$ is denoted by $\mathcal{MF}_{G}(R)= \big\{\{x,y\}:x,y \text{~is~an~mf-pair~of~} R \text{~in~} G\big\}$.
\end{definition}

Indeed, if $\{x,y\}$ is an mf-pair of $R$ in $G$, this implies that the information flow between $x$ and $y$ cannot be blocked by any subset of $R$. This leads to the notion of collapsibility in conditional independence.
There exist three types of collapsibility in graphical models below:

(i)  A CBN  $(G,\P(G))$ is conditional independence collapsible (CI-collapsible)  onto $A$ (or over $V\backslash A$) if its structural information $G$ satisfies ${I}(G)_A = I(G_A)$. 

(ii) A CBN  $(G,\P(G))$  is  model collapsible onto $A$ (or over $V\backslash A$) if  $\mathcal{P}(G)_A=\P(G_A)$. 

(iii)  A CBN  $(G,\P(G))$ is estimate collapsible onto $A$ (or over  $V\backslash A$) if $\hat{\mathcal{P}}(G_A) = \hat{\mathcal{P}}(G)_A$, where  $\hat{\mathcal{P}}(x_{V})$ is the maximum likelihood estimates (MLEs) of $\mathcal{P}(x_V)$. 

In this work, we consider the estimate collapsibility in CBNs under the following assumptions.

Assumption 1: Every distribution considered in a distribution family follows either a multinomial distribution or a Gaussian distribution. This assumption ensures that marginal distributions are closed under the corresponding joint distributions. 

Assumption 2: Every joint distribution is positive. It requires that all CBNs considered satisfy the graphoid axioms.

Under Assumptions 1-2, the non-triviality condition is satisfied~\citep{xie2009}, as shown below. A CBN \(\mathcal{B} = ({G}, \mathcal{P}({G}))\) is non-trivial if the conditions hold: 	(i) for every DAG \({G}'\) in the Markov equivalence class of \({G}\), the parameters of \(\mathcal{P}({G})\) are variation independent with respect to \({G}'\),	
(ii) there exists at least one distribution \(P \in \mathcal{P}({G})\) that is faithful to \({G}\), and 	
(iii) for each factor $P(x_v \mid pa_{G}(x_v))$ in Equation~(\ref{formula}), if a vertex $v$ has non-adjacent parents $u$ and $w$, there exist observations such that:
$\hat{P}_{{G}_{V \setminus \{u\}}}(x_v \mid pa_{{G}}(x_v) \setminus \{u\}) \neq \int \hat{P}(x_v \mid pa_{{G}}(x_v)) \hat{P}(u) \, du.	$

{ Based on the above analysis, we will estimate causal effects using the collapsibility and the back-door criterion in the following sections. } 
\section{Identify the Strong d-Convex Hulls}\label{sec3}

In this section, we provide theoretical characterizations of minimal collapsible sets and develop an efficient algorithm for identifying them.
\subsection{ Properties of Strong d-Convex Hulls}
We propose a novel concept, strong d-convex hulls, that extends the notion of d-convex hulls introduced by {\cite{heng2024}} to investigate model collapsibility in CBNs. As shown by \cite{xie2009}, estimate collapsibility implies model collapsibility, but the converse does not necessarily hold. To seek intuitive conditions for the estimate collapsibility, we begin with the notion of a linearly ordered set.
\begin{definition}\label{def:partially-complete}
	For a DAG $G=(V,E)$ and $R\subseteq V$, a vertex $v$ is linearly ordered with respect to $R$ if any two distinct vertices in $pa_{G}(v)$ are adjacent, except when both vertices belong to $R$. A subset $A\subseteq V$ is linearly ordered with respect to $R$ if every $w\in A$ is linearly ordered with respect to $R$.
\end{definition}

\begin{definition}\label{def:sc}
	Let $G=(V,E)$ be a DAG. A subset $R\subseteq V$ is said to be strongly d-convex if it satisfies the following two conditions: 	(i)  there is no inducing path for $R$ in $G$, and 	(ii) ${Ch}_{G}(M)\cap {An}_{G}(R)$ is linearly ordered with respect to $R$, where $M=V\setminus R$, ${Ch}_{G}(M)={ch}_{G}(M)\cup M$ and ${An}_{G}(R)={an}_{G}(R)\cup R$.
\end{definition}

If a subset $R$  only satisfies the condition (i) in Definition \ref{def:sc}, it is said to be d-convex in~\cite{heng2024}, a property that ensures CI-collapsibility and model collapsibility in CBNs but not for estimate collapsibility. This motivates us to investigate the graphical criterion for the estimate collapsibility in CBNs. The proposed strong d-convexity is generally a strictly stronger condition than d-convexity. The following result reveals its benefits in CBNs.



\begin{lemma}\label{thm-c-remo}{\rm
	Suppose $\mathcal{B} = (G, \mathcal{P}(G))$ is a CBN, where $G = (V, E)$ and $R\subseteq V$.   {Under the assumption of non-triviality}, the statements are equivalent:
	(i) $\hat{P}(G_R) = \hat{P}(G)_R$,
	(ii) there exists a graph $G^\prime$ in the Markov equivalence class of $G$ such that $M = V\backslash R$ is a terminal  set in $G^\prime$, and
	(iii) $R$ is a strong d-convex set in $G$.}
\end{lemma}
\begin{proof}
The equivalence of (i) and (iii) is implied from {\citep[Theorem~3]{deng}}. The equivalence of (i) and (ii) can be easily derived from \cite{xie2009}.
\end{proof}

Lemma~\ref{thm-c-remo} implies that, for computing the MLEs of the marginal distributions of several variables, it suffices to find the strong d-convex hull that contains these variables, and then reparameterize based on the induced structure of the strong d-convex hull to obtain a fast solution.  To identify strong d-convex sets, we first establish the equivalent graphical characterizations of d-convex sets.

\begin{lemma}\label{lem-mf}{\rm
	Let $G = (V, E)$ be a DAG and $R \subseteq V$, the following statements are equivalent:
	(i) $R$ is a d-convex set, 
	(ii) $\mathcal{MF}_{G}(R)=\emptyset$, and
	(iii) there are no inducing structures of $R$ in $G_{R}$.}
\end{lemma}
\begin{proof}
	The equivalence of (i) and (ii) is implied by the definition of d-convex~\citep{heng2024}. We then prove that (i) implies (iii). Without loss of generality, suppose that $l_{xy}$ is an inducing path of $R$. We give the proof by considering the following two cases:
	
	Case i. $V^c(l_{xy}) = \emptyset$. Suppose there exists a vertex $w \in V(l_{xy}) \setminus R$. Traverse $l_{xy}$ from $w$ in both directions until the first vertices in $R$ are encountered in each direction. If these two vertices $u$ and $v$ are non-adjacent, this defines a subpath $l_{uv} \subseteq l_{xy}$ from $u$ to $v$ that contains $w$. Clearly, $l_{uv}$ constitutes an inducing path relative to $R$. Thus, we have $u \ndep v|\emptyset$, $\{u, v\}$ forms an mf-pair, which is a contradiction.
	
	Case ii. $V^c(l_{xy}) \neq \emptyset$. Assume there exists a vertex $w \notin R$ in the inducing structure. Similarly, traverse $l_{xy}$ from $w$ in both directions to locate the first vertices $u, v \in R$ encountered, yielding a subpath $l_{uv} \subseteq l_{xy}$ from $u$ to $v$ that passes through $w$. If $V^c(l_{uv}) = \emptyset$, then $l_{uv}$ is an inducing path whose inducing structure is not contained in $G_R$, leading to a contradiction. Otherwise, let $c$ be a collider on $l_{uv}$ that is closest to $v$. Then, there exists a vertex $v'$ on the directed path from $c$ to $x$ such that the union of the shortest subpath ${l}_{cy}$ and $l_{cx}$ (the segment of $l_{xy}$ from $c$ to $x$ via $v'$) forms an inducing path relative to $R$, again resulting in a contradiction. The converse implication, (iii) $\Rightarrow$ (i), follows straightforwardly by contradiction.
\end{proof}
\begin{remark} 
	This type of convexity in directed graphs generalizes the convexity defined in~\cite{Pfaltz1971}, which requires that a directed path with its endpoints in $R$ must be an inducing path of $R$ if these endpoints are not adjacent.  A convex subgraph defined in an undirected graph is a special case of d-convex subgraphs. In undirected graphs, any path $l_{xy}$ satisfies $R\cap V^o(l_{xy})\subseteq V^c(l_{xy})=\emptyset$, meaning that all internal vertices of $l_{xy}$ must be in $V\setminus R$ and  thus form a minimal $R-R$ path. Furthermore, $R$ is a d-convex set if and only if $V\setminus R$ is a t-removable set, a concept introduced by \cite{deng} in directed graphical models.
\end{remark}

Beyond the equivalent characterizations of d-convexity, several stronger and more useful properties of strongly d-convexity can be established. We begin by introducing the inheritance property.

\begin{lemma}\label{pro-heredity}{\rm
	Let $G=(V, E)$ be a DAG and $H_2$ be a strong d-convex set in $G$. For any subset $H_1 \subseteq H_2$,  $H_1$ is strongly d-convex in $G$ if and only if $H_1$ is strongly d-convex in the induced subgraph $G_{H_2}$.}
\end{lemma}
\begin{proof}
	This follows directly from $\hat{P}(G_{H_1}) = \hat{P}(G_{H_2})_{H_1} = \hat{P}(G)_{H_1}$ by  Lemma~\ref{thm-c-remo}.
\end{proof}

Define the intersection of two subgraphs $G_1=(V_1, E_1)$ and $G_2=(V_2, E_2)$ as $G_1\cap G_2=(V_1\cap V_2, E_1\cap E_2)$. Then, an arbitrary intersection of d-convex subgraphs is strongly d-convex. As with normal subgroups, the union of strong d-convex subgraphs need not be strongly d-convex. 

\begin{lemma}\label{lem-closure}{\rm
	Let ${G} = (V, E)$ be a DAG and $R \subseteq V$. Let $\{H_i\}_{i=1}^n$ be a collection of strong d-convex subsets of $V$ containing $R$. Then, their intersection $H = \cap_{i=1}^n H_i$ is strongly d-convex.}
\end{lemma}
\begin{proof}
	By Lemma~\ref{thm-c-remo}, for each $i\in \{1, \cdots n\}$, the set $V \setminus H_i$ is a terminal set. Let $H = \cap_{i=1}^n H_i$, by De Morgan's laws, we have  $V \setminus H =  (V \setminus \bigcap_{i=1}^n H_i) =\bigcup_{i=1}^n (V \setminus H_i)$ is a terminal set, that is, $H$ is strongly d-convex.
\end{proof} 

Based on the closure property, a strong d-convex subgraph may contain smaller strong d-convex subgraphs, making it valuable to identify the smallest strong d-convex subgraph containing these target variables. Such a subgraph facilitates the construction of a local model with the minimum cardinality. Due to the statistical properties of strong d-convex subgraphs, inference for the target variables based on this local model is fully consistent with that derived from the full model \citep{xie2009}, highlighting the importance of collapsibility in estimation.

\begin{definition}\label{sdch}
	Let $G=(V,E)$ be a DAG and $R \subseteq V$. We say that an induced subgraph $G_H$ is a strong d-convex hull of $R$ if and only if $G_{H'}$ is not a strong d-convex subgraph for all $H'$ satisfying $R\subseteq H'\subsetneq H$.    
\end{definition}
The strong d-convex hull of $R$ is the minimal strong d-convex subgraph containing $R$. For simplicity, given a variable set $R$ in a DAG $G$, we use, respectively, $ch(R)$ and $sch(R)$, to represent its d-convex hull and strong d-convex hull in $G$. The following consequence shows their existence and uniqueness.	
\begin{lemma}\label{lem-uniquenss-ch}{\rm
	For a subset $R \subseteq V$, $ch(R)$ and $sch(R)$ exist and are unique. }
\end{lemma}
\begin{proof}The existence easily follows from Lemma~\ref{lem-closure}. Suppose that $H_1, H_2$ are two distinct strong d-convex hulls of $R$. This implies that $H_1\cap H_2$ is also a strong d-convex hull containing $R$, which is smaller than $H_1$ and $H_2$, a contradiction.
\end{proof}

\begin{definition} \label{def:mcs}
	Let $\mathcal{B}=(G,\mathcal{P}(G))$ be a CBN and $R\subseteq V$. The subset $H$ containing $R$ is the minimal collapsible set if it satisfies
	(i) $\hat{P}_{G_H}(x_H)=\hat{P}(x_H)$, and
	(ii) $\hat{P}_{G_{H'}}(x_{H'})\neq \hat{P}(x_{H'})$ for all $P(x_V)\in \mathcal{P}(G)$ and any subset $H'$ satisfying $R\subseteq H' \subsetneq H$.
\end{definition}

For a given CBN $\mathcal{B}$ and a target variable set $R$, Definition~\ref{def:mcs} guarantees that inference on $R$ only requires the local subnetwork $\mathcal{B}_H$ over the minimal collapsible set $H$ containing $R$. In other words, the set $H$ is the smallest set that preserves the original inference outcomes for $R$ after collapsing.

\begin{theorem}\label{cor-new}{\rm
	Given a CBN $\mathcal{B}=(G,\mathcal{P}(G))$, where $R\subseteq H\subseteq V$, $H$ is the minimal collapsible set containing $R$ if and only if $G_H$ is the strong  d-convex hull of $R$.}
\end{theorem}
\begin{proof}
	It directly follows from Lemma~\ref{thm-c-remo} and Definition~\ref{def:mcs}.
\end{proof}

Theorem \ref{cor-new} reveals that identifying the minimal collapsible set reduces to computing the strong d-convex hull of $R$, effectively transforming a challenging statistical problem into a tractable graph-theoretic task. 

\subsection{Algorithms for Identifying Strong d-Convex Hulls}
In this subsection, we provide equivalent conditions of strong d-convex hulls, termed minimal inducing structures, and develop algorithms for computing the strong d-convex hulls in DAGs. 

In undirected graphs, an inducing path is a path in which every internal vertex has degree two with respect to the path, or, more generally, a path for which no shortcuts exist that bypass the internal vertices while preserving the connection between the endpoints. A convex set \( H \) containing \( R \) is required to include all vertices traversed by any inducing path between nodes in \( R \). This property has been exploited by \cite{heng2023} to design an iterative algorithm for computing the convex hulls. A natural question then arises: does an analogous property hold for DAGs? The following counterexample demonstrates that, in general, this is not the case.

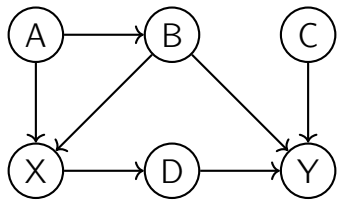
\begin{figure}[pos=htbp]
	\centering
	\begin{tikzpicture}[
		scale=0.9,
		transform shape,
		nodeR/.style = {fill=red!15,draw=red!70!black}, 
		nodeMB/.style = {fill=red!15,draw=black, thick}, 
		every node/.style={
			draw, 
			circle, 
			thick, 
			minimum size=8mm,
			inner sep=0pt, 
			align=center,
			font = \Large
		}]
		\node (X) at (0,0) {X};
		\node (D) at (2,0) {D};
		\node (Y) at (4,0) {Y};
		\node (A) at (0,2) {A};
		\node (B) at (2,2) {B};
		\node (C) at (4,2) {C};
		\draw [->, thick] (X) -- (D);  
		\draw [->, thick] (D) -- (Y);  
		\draw [->, thick] (A) -- (X);  
		\draw [->, thick] (A) -- (B);  
		\draw [->, thick] (B) -- (X);  
		\draw [->, thick] (B) -- (Y);  
		\draw [->, thick] (C) -- (Y);       
	\end{tikzpicture}
	\caption{$G$ with  $R=\{ X, Y\}$.}
	\label{fig:example}
\end{figure}
\begin{example}
	Consider the DAG $G$ with \( R = \{X, Y\} \) in Figure~\ref{fig:example}. There are 3 inducing paths between $X$ and $Y$. Applying the undirected graph principle, we take the union of all vertices on { $\mathrm{IS}_G{(l_{XY})}$.} Then, $\{X, A, B, D, Y\}$ is d-convex. However, removing node $A$ while preserving d-convexity yields a smaller d-convex set $\{X, B, D, Y\}$.
\end{example}
This example demonstrates that constructing d-convex hulls in DAGs cannot be achieved by simply generalizing the method for undirected graphs. Specifically, some vertices may be eliminated without breaking d-convexity.
We then present an equivalent characterization of the d-convex hull in DAGs.
\begin{definition}\label{def:misc}
	Let $l_{xy}$ be an inducing path of $R$ in $G$. The {$\mathrm{IS}_G(l_{xy})$} is minimal if its corresponding inducing path $l_{xy}$ is minimal in length, denoted by {$\mathrm{MIS}_{G}(l_{xy})$.}
\end{definition}
\begin{example}[Continued]
	Let $R = \{X, Y\}$ as shown in Figure~\ref{fig:example}. The path $l_{XY}=(X, D, Y)$ is an inducing path of length 3, while  $l'_{XY}=(X, A, B, Y)$ is an inducing path of length 4. Hence, $l_{XY}=(X, D, Y)$ is minimal and naturally possesses a minimal inducing structure.
\end{example}

\begin{lemma}\label{lem-d-convex}{\rm
	Let ${G}=(V,E)$ be a DAG and $R \subseteq H \subseteq V$. For any two non-adjacent vertices $r_1, r_2 \in R$, if $H$ is d-convex, we have $V^o(\mathrm{MIS}_G(l_{r_1r_2})) \subseteq H$.}
\end{lemma}
\begin{proof}
	Suppose there exists a vertex $v \in V^o( \mathrm{MIS}_G(l_{r_1r_2}))$ and $v \notin H$. From the definition~\ref{def:misc},  the inducing path relative to $ \mathrm{MIS}_G{(l_{r_1r_2})}$ contains no colliders. Traverse $l_{r_1r_2}$ from $v$ in both directions until the first vertices in $R$ are encountered in each direction, denote these vertices as $u$ and $w$. We have $u \ndep w |\emptyset$, hence $\{u, w\}$ forms an $mf$-pair, yielding  a contradiction.
\end{proof}

Considering the relationship between d-convexity and strong d-convexity, we derive the condition for strong d-convex hulls.
\begin{lemma}\label{pro-convex hull}{\rm
	Let ${G}=(V,E)$ be a DAG, $R \subseteq H \subseteq V$, and $r_1, r_2 \in R$ be  non-adjacent. 
	If ${G}_H$ is strongly d-convex, then  the following hold:	
	(i) for any $ \mathrm{MIS}_G(l_{r_1r_2})$ in $G$, we have 
	$V^o( \mathrm{MIS}_G(l_{r_1r_2})) \subseteq H$, and
	(ii) for any $w \in Ch_{G}(M) \cap An_{G}(R)$, with $M = V\setminus R$, if $x, y \in pa_{G}(w)$ are non-adjacent, then both $x$ and $y$ belong to $H$.}
\end{lemma}

\begin{proof}
	By the Definition~\ref{def:misc}, the $V^o( \mathrm{MIS}_G(l_{r_1r_2}))$ corresponds to the shortest inducing path. Suppose there exists a node 
	$v \in V^o(l_{r_1r_2})$ on such a  path 
	$l_{r_1r_2}$ in $({G}_{An_{{G}}(\{r_1, r_2\})})^m$ with 
	$r_1, r_2 \in R$, but $v \notin H$.  
	Then Lemma~\ref{lem-d-convex} implies an inducing path 
	within $R$, contradicting the d-convexity of $H$.  Suppose there exists 
	$w \in Ch_{G}(M) \cap An_{G}(R)$ with non-adjacent parents 
	$x, y \in pa_{G}(w)$, and at least one parent not in $H$. Then
	$w \in ch_{G}(V \setminus H) \cap H$.  
	Since $
	Ch_{G}(V \setminus H) \cap An_{G}(H)
	= \big( ch_{G}(V \setminus H) \cap H \big) 
	\cup \big( An_{G}(H) \cap (V \setminus H) \big),
	$
	we have $w \in Ch_{G}(V \setminus H) \cap An_{G}(H)$, contradicting the strong d-convexity of $H$. 
\end{proof}

Lemma \ref{pro-convex hull} highlights a connection between the construction of strong d-convex hulls and minimal inducing structures, particularly inducing paths. Motivated by this, we provide a more intuitive graph-theoretic characterization of $\mathrm{MIS}$ within the d-convexity framework.
\begin{lemma}\label{lem-convex hull}{\rm
	Let ${G}=(V,E)$ be a DAG, $R \subseteq H \subseteq V$, and let $r_1, r_2 \in R$ be  non-adjacent. 
	If ${G}_H$ is (strongly) d-convex, then  the following statements are equivalent: 	
	(i) for any $ \mathrm{MIS}_G(l_{r_1r_2})$ of $R$ in $G$,
	$V^o( \mathrm{MIS}_G(l_{r_1r_2})) \subseteq H$.	
	(ii) every shortest path $l_{r_1r_2}$ in $({G}_{An_{{G}}(\{r_1, r_2\})})^m$ satisfies $V^o(l_{r_1r_2}) \subseteq H$.}
\end{lemma}

\begin{proof}
Consider an $ \mathrm{MIS}_G(l_{r_1r_2})$ of $R$ connecting $r_1$ and $r_2$ in $G$.  For any $v_i\in V^o(l_{r_1r_2})$, we examine two cases:
	
	(a) If $v_i\in V^{c}(l_{r_1r_2})$, then by Definition \ref{Inducing path},  $v_i\in an_{G}(\{r_1,r_2\})$.
	
	(b) If $v_i\notin V^c(l_{r_1r_2})$,  $\operatorname{MIS}_G(l_{r_1r_2})$
	must contain a path without colliders,
	on which every vertex is an ancestor of either 
	$r_1$ or $r_2$.
	
	From (a) and (b),  $V^{o}(l_{r_1r_2})$ lies in either $an_{G}(r_1)$ or $an_{G}(r_2)$. Hence, $l_{r_1r_2}$ is a path in $(G_{An_{G}(\{r_1,r_2\})})^{m}$ with $V^{o}(l_{r_1r_2})\subseteq M$. 		
	Replacing moral edges on $l_{r_1r_2}$ with V-structures in $G_{An_{G}(\{r_1,r_2\})}$ recovers the $\mathrm{MIS}_{G}(l_{r_1r_2})$.
\end{proof}

Since any d-convex set must contain all the vertices of every $\mathrm{MIS}$ between its non-adjacent vertices, the d-convex hull is obtained by repeatedly absorbing these vertices. The computation proceeds iteratively in two phases. First, we collect all the vertices on the shortest inducing paths between non-adjacent pairs in the current target set, forming an initial set that contains all the inducing paths of 
$R$, denoted by {AIP} (see Algorithm~\ref{Alg-ip}).
We then iteratively enlarge this set by reapplying the same procedure until no new vertices are added. We will show that the resulting set is the d-convex hull $H$ of $R$ (see Algorithm~\ref{Alg-d}).

\begin{breakablealgorithm}
	\caption{Collecting Vertices on MIS (CVM(${G}, R$))}\label{Alg-ip}
	\begin{algorithmic}[1]
		\Require A  DAG ${G} = (V, E)$ and a variable set $R \subseteq V$.
		\Ensure the vertices on the shortest inducing paths of each non-adjacent vertex in $R$.
		\State Initialize: $M \gets V \setminus R$, $ \mathrm{AIP} \leftarrow \emptyset$.
		\State Compute: $R_1 = mb_{G}(M)\cap R$ and $\mathrm{nadj}$=$ \big \{\{r_1, r_2\}: r_1\overset{G} \nsim r_2, r_1, r_2\in R_1\big \}$
		\For{$r_1, r_2\in \mathrm{nadj}$}
		\State Compute $G^* = (G_{An_{G}(\{r_1, r_2\})})^m$ and $G_1 = G^*_{\{r_1, r_2\}\cup M}$
		\State $\mathrm{I} \gets $ vertices on the shortest path of $r_1, r_2$ in $G_1$.
		\State $\mathrm{AIP} \gets \mathrm{AIP} \cup \mathrm{I}$
		\EndFor
		\State \Return $\mathrm{AIP}$.
	\end{algorithmic}
\end{breakablealgorithm}

\begin{breakablealgorithm}
	\caption{Identifying d-Convex Hull Algorithm (ICHA(${G}, R$))}\label{Alg-d}
	\begin{algorithmic}[1]
		\Require A  DAG ${G} = (V, E)$ and a variable set $R \subseteq V$.
		\Ensure A d-convex hull $H$ containing $R$.
		\State Initialize: $H\leftarrow R, \mathrm{AIP} \leftarrow R$.
		\While{${\mathrm{AIP}}$}
		\State $\mathrm{AIP} = \mathrm{CVM}(G, H)$
		\State $H = H \cup \mathrm{AIP}$
		\EndWhile
		\State \Return ${H}$.
	\end{algorithmic}
\end{breakablealgorithm}

\begin{proposition}
	The time complexities of the CVM and ICHA algorithms are  \( O(k^2 \times |V|^2) \) and  $O(k^2 \times |V|^3)$, respectively. The space complexity for both algorithms is bounded by $O(|V|^2)$.
\end{proposition}

\begin{proof}
	The complexity of CVM is dominated by identifying inducing paths between 
	non-adjacent nodes. For each pair, constructing the subgraph $G^*$ and 
	performing BFS both require at most $O(|V|^2)$ time. Since there are at most 
	$O(k^2)$ non-adjacent pairs, the overall complexity of CVM is 
	$O(k^2 |V|^2)$, where $k = |R_1|$. 
	The ICHA algorithm iteratively invokes CVM, and the while loop runs at most 
	$|V|$ times. Therefore, the total complexity of ICHA is $O(k^2 |V|^3)$.
	For space complexity, both CVM and ICHA are bounded by $O(|V|^2)$, which is 
	dominated by storing the moralized graph $G^*$.
\end{proof}

Based on the d-convex hull, we can turn it into the strong d-convex hull by making it linearly ordered, which inspires us to propose the following Algorithm~\ref{Alg-c}. The following theorem establishes the theoretical correctness of Algorithm~\ref{Alg-c} for computing the strong d-convex hull.

\begin{theorem}\label{the:d-con}{\rm
		The set $H$ is a strong d-convex hull of $R$ in Algorithm~\ref{Alg-c}.}
\end{theorem}

\begin{proof}
	Let $H'$ be the strong d-convex hull of $R$. We prove that $H = H^{'}$. Let $l_{r_1r_2}$ be a shortest inducing path connecting two non-adjacent nodes $r_1, r_2\in R$. From Lemma~\ref{pro-convex hull}, we know that $V^o({\mathrm{MIS}_{G}(l_{r_1r_2})})\subseteq H'$. Thus, the nodes absorbed in the first iteration of  Algorithm~\ref{Alg-c} are already contained in $H'$. If any vertex $w \in Ch_G(M) \cap An_G(R)$ violates the linear ordering condition, its non-adjacent parents must be included in any strong d-convex set, and thus are already contained in $H'$. By iterating this process, every node absorbed in each step is included in $H'$, leading to $H\subseteq {H^{'}}$. 	Next, we prove that ${H}$ is a strong d-convex set of ${G}$. By construction, Algorithm~\ref{Alg-c} iteratively absorbs all vertices on shortest inducing paths between non-adjacent pairs and the nodes violating the linear ordering in the current set until no new vertices can be added. Upon termination, there are no remaining inducing paths between non-adjacent nodes within $H$ that can introduce new vertices, and $H$ satisfies the linear ordering condition. Thus $H$ is strongly d-convex and contains $R$. Since ${H^{'}}$ is the strong d-convex hull containing $R$, we have ${H^{'}}\subseteq {H}$, it follows that $H = {H^{'}}$.
\end{proof}

\begin{breakablealgorithm}
	\caption{Identifying Strong d-Convex Hull Algorithm (ISCHA(${G}, R$))}\label{Alg-c}
	\begin{algorithmic}[1]
		\Require A DAG ${G} = (V, E)$ and a variable set $R \subseteq V$.
		\Ensure A strong d-convex hull $H$ containing the variable set $R$.
		\State Initialize: $H \gets R$.
		\Repeat
		\State $H \gets \mathrm{ICHA}(G, H)$ 
		\State $PA \gets \{pa_G(w) \mid w \in {Ch}_{G}(V\backslash H)\cap {An}_{G}(H), w \text{ is not linearly ordered}\}$
		\State $H \gets H \cup PA$
		\Until{$PA = \emptyset$} 
		\State \Return $H$
	\end{algorithmic}
\end{breakablealgorithm}

\begin{proposition}
	The time complexity of the ISCHA algorithm is $O(k^2 \times |V|^4)$.
\end{proposition}
\begin{proof}	
	For the ISCHA algorithm, verifying linear ordering requires $O(|V| + |E|)$. Since each iteration of the main loop incurs a complexity of $O(k^2 \times |V|^3)$ and the loop runs at most $O(|V|)$ times, the total time complexity of the algorithm is $O(k^2 \times |V|^4)$. Regarding space complexity, the memory usage is dominated by the ICHA procedure, which requires $O(|V|^2)$ space to store the moralized graph. The additional sets $H$ and $PA$ require at most $O(|V|)$ space, thus maintaining the overall space complexity at $O(|V|^2)$.
\end{proof}

To be clear, we work out an example to illustrate these proposed Algorithms~\ref{Alg-ip}-\ref{Alg-c}. 
\begin{example}[Continued]\label{ex:sdc}
	Consider the DAG with target set $R=\{X, Y\}$ in Figure~\ref{fig:subexample}(a). Since $mb_G(M)\cap R=\{X, Y\}$, we construct the moralized graph $G^* = (G_{\{A, B, C, D, X, Y \}})^m$, see Figure~\ref{fig:subexample}(b).
	In $G^*$, the shortest path between $X$ and $Y$ whose internal vertices lie in $V\setminus R$ 
	is $(X,B,Y)$. 
	Therefore, the initial  $\mathrm{AIP}$ set is $\{X, B, Y\}$, from Algorithm~\ref{Alg-ip}. By  Algorithm~\ref{Alg-d} and  the output $\mathrm{AIP}$, we update $H = \{X, B, Y\}$. 
	Call CVM($G, H$) and obtain $\mathrm{AIP} = \{X, B, D, Y\}$. 
	The algorithm terminates, returning the d-convex hull $\{ X, B, D, Y\}$. Through Algorithm~\ref{Alg-c}, we compute the strong d-convex hull. 
	Starting from $R=\{X, Y\}$, the procedure ICHA($G, R$) returns 
	$H=\{X, B, D, Y\}.$ We then examine condition (ii) in Definition~\ref{def:sc}. Node $Y$ does not satisfy this condition, the parent set is $pa_G(Y)=\{B, C, D\}.$
	Updating $H$ accordingly gives $H=\{X, B, C, D, Y\}.$ Re-evaluating condition (ii), no further violations occur, and the algorithm terminates. 
	Thus, the strong d-convex hull is $\{X, B, C, D, Y\}.$ Figure~\ref{fig:subexample} presents an illustration of Algorithms~\ref{Alg-ip}--\ref{Alg-c} applied to a DAG with target set $R=\{X, Y\}$.
\end{example}

\begin{figure}[pos=htbp]
	\centering
	\begin{subfigure}[t]{0.3\textwidth}
		\centering
		\begin{tikzpicture}[
			scale=0.9,
			transform shape,
			nodeR/.style = {fill=red!15,draw=red!70!black}, 
			nodeMB/.style = {fill=red!15,draw=black, thick}, 
			every node/.style={
				draw, 
				circle, 
				thick, 
				minimum size=8mm,
				inner sep=0pt, 
				align=center,
				font = \Large
			}]
			\node [nodeR](X) at (0,0) {X};
			\node (D) at (2,0) {D};
			\node [nodeR](Y) at (4,0) {Y};
			\node (A) at (0,2) {A};
			\node (B) at (2,2) {B};
			\node (C) at (4,2) {C};
			\draw [->, thick] (X) -- (D);  
			\draw [->, thick] (D) -- (Y);  
			\draw [->, thick] (A) -- (X);  
			\draw [->, thick] (A) -- (B);  
			\draw [->, thick] (B) -- (X);  
			\draw [->, thick] (B) -- (Y);  
			\draw [->, thick] (C) -- (Y);    
		\end{tikzpicture}
		\caption{$G$, with  $R=\{ X, Y\}$}
		\label{fig:subexample1a}
	\end{subfigure}
	\hfill
	\begin{subfigure}[t]{0.3\textwidth}
		\centering
		\begin{tikzpicture}[
			scale=0.9,
			transform shape,
			nodeR/.style = {fill=red!15, draw=red!70!black}, 
			nodeMB/.style = {fill=red!15,draw=black, thick},
			nodeP2/.style = {fill=blue!15, draw=blue!70!black, dotted},
			nodePath/.style = {fill=green!15, draw=green!70!black, dotted},
			nodep/.style = {fill=yellow!20, draw=yellow!80!black, dashed},
			every node/.style={
				draw, 
				circle, 
				thick, 
				minimum size=8mm,
				inner sep=0pt, 
				align=center,
				font = \Large
			}]
			\node [nodeR](X) at (0,0) {X};
			\node [nodeP2](D) at (2,0) {D};
			\node [nodeR](Y) at (4,0) {Y};
			\node (A) at (0,2) {A};
			\node [nodePath](B) at (2,2) {B};
			\node (C) at (4,2) {C};
			\draw [-, thick] (X) -- (D);  
			\draw [-, thick] (D) -- (Y);  
			\draw [-, thick] (A) -- (X);  
			\draw [-, thick] (A) -- (B);  
			\draw [-, thick] (B) -- (X);  
			\draw [-, thick] (B) -- (Y);  
			\draw [-, thick] (C) -- (Y);    
			\draw [-, thick] (B) -- (C);  
			\draw [-, thick] (C) -- (D);    
			\draw [-, thick] (B) -- (D);
		\end{tikzpicture}
		\caption{$G^*$, with nodes lying on the shortest path, using ICHA }
			\label{fig:subexample1b}
	\end{subfigure}
	\hfill
	\begin{subfigure}[t]{0.33\textwidth}
		\centering
		\begin{tikzpicture}[
			scale=0.9,
			transform shape,
			nodeR/.style = {fill=red!15,draw=red!70!black}, 
			nodeMB/.style = {fill=red!15,draw=black, thick}, 
			nodePath/.style = {fill=green!15, draw=green!70!black, dotted},
			nodeP2/.style = {fill=blue!15, draw=blue!70!black, dotted},
			nodep/.style = {fill=yellow!20, draw=yellow!80!black, dashed},
			every node/.style={
				draw, 
				circle, 
				thick, 
				minimum size=8mm,
				inner sep=0pt, 
				align=center,
				font = \Large
			}]
			\node [nodeR](X) at (0,0) {X};
			\node [nodeP2](D) at (2,0) {D};
			\node [nodeR](Y) at (4,0) {Y};
			\node (A) at (0,2) {A};
			\node [nodePath](B) at (2,2) {B};
			\node [nodep](C) at (4,2) {C};
			\draw [->, thick] (X) -- (D);  
			\draw [->, thick] (D) -- (Y);  
			\draw [->, thick] (A) -- (X);  
			\draw [->, thick] (A) -- (B);  
			\draw [->, thick] (B) -- (X);  
			\draw [->, thick] (B) -- (Y);  
			\draw [->, thick] (C) -- (Y);    
		\end{tikzpicture}
		\caption{$G$, with nodes violating condition (ii), using ISCHA}
		\label{fig:subexample1c}
	\end{subfigure}
	
	\caption{Illustration of Algorithms~\ref{Alg-ip}--\ref{Alg-c}.}
	\label{fig:subexample}
\end{figure}
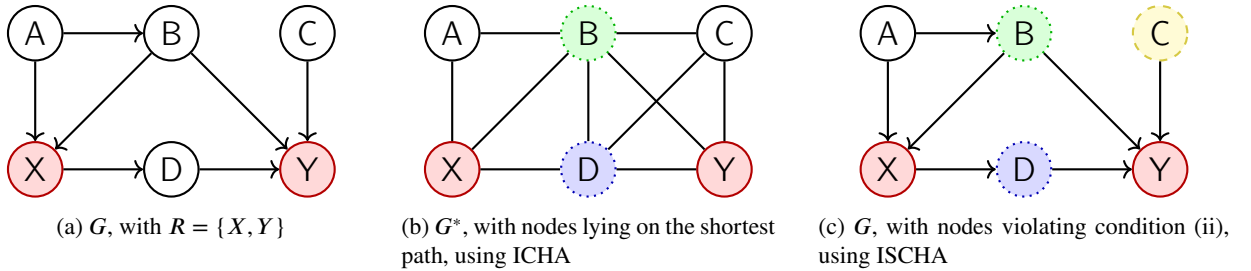

\section{Estimate Collapsibility for CPDAGs}\label{sec4}

Given the commutativity of MLEs and marginalization for a fixed DAG and target set, a natural question arises: can this notion of estimate collapsibility be extended to causal effect estimation? To address this, we propose a new concept of estimate collapsibility for CPDAGs, and combine the graph reduction procedure with the IDA framework. We first present several properties of DAGs.
\begin{definition}[]\label{def:causal_collapsibility}
	Let $X$ and $Y \notin pa_G(X)$ be two non-adjacent variables in a DAG $G$, and let $R \subseteq V$ contain $X$ and $Y$. We say that $G$ is causal estimate collapsible onto $R$ if for every non-empty valid back-door adjustment set $Z$ in $G_R$, $$\hat{P}(Y \mid do(X=x)) = \sum_{z} \hat{P}_{G_R}(Y \mid X=x, Z=z) \hat{P}_{G_R}(Z=z).$$
\end{definition}

\begin{theorem}\label{th:validity}{\rm
 Let $X$ and $Y \notin pa_G(X)$ be two non-adjacent vertices in a DAG $G$, and let $Z$ be a non-empty back-door adjustment set in $G_R$, where $R\subseteq V$. Then, $G$ is causal estimate collapsible onto $R$ if  $R$ is a strong d-convex hull of $X$ and $Y$.}
\end{theorem}
\begin{proof}  {  $Z$ is a valid back-door adjustment set in $G$ only if it satisfies two conditions: (1)  $Z$ blocks all back-door paths between $X$ and $Y$ in $G$, and (2)  $Z$ contains no descendants of $X$ in $G$. We consider two cases:
	
(i) Back-door path without colliders.
Such a path is an inducing path relative to $R$. 
By construction of the strong $d$-convex hull (Algorithm~\ref{Alg-c}), it is preserved in $G_R$ and thus blocked by $Z$ in $G_R$, hence also in $G$.

(ii) Back-door path $\pi$ containing at least one collider.
Suppose, for contradiction, that conditioning on $Z$ activates $\pi$ in $G$. 
Since $Z \subseteq R$, the activation of $\pi$ implies that every collider on $\pi$ is either in $R$ or has a descendant in $R$. 
If a collider or its descendant is absorbed into $R$,  by the construction of the strong $d$-convex hull, there must exist an inducing path between $X$ and $Y$ passing through this node. By Lemma~\ref{lem-convex hull}, this node is an ancestor of either $X$ or $Y$. 
It follows that $\pi$ itself forms an inducing path relative to $R$. 
According to the ISCHA algorithm (Algorithm~\ref{Alg-c}), the inducing path is preserved in $G_R$. 
Therefore, $\pi$ is active between $X$ and $Y$ given $Z$ in $G_R$, which leads to a contradiction. }

Moreover, if some $z \in Z$ is a descendant of $X$ in $G$, the same holds in $G_R$, contradicting the validity of $Z$ as a back-door set. Therefore, $Z$ is also a valid back-door adjustment set in $G$, and the collapsibility equality follows.
\end{proof}

\begin{example}
Consider the DAG $G$ shown in {Figure~\ref{fig:subexample}(a).} The strong d-convex hull of $X$ and $Y$ is $R = \{X, Y, B, C, D\}$ { (Figure~\ref{fig:subexample}(c))}. The post-intervention distribution can be computed directly on the induced subgraph $G_R$,
	$$ \begin{aligned}
		\hat{P}(Y \mid do(X=x)) &= \sum_{a, b} \hat{P}(Y \mid X=x, A=a, B=b)\hat{P}(A=a, B=b)\\
		&=\sum_{b} \hat{P}_{G_R}(Y \mid X=x, B=b)\hat{P}_{G_R}(B=b).
	\end{aligned} $$
In the full graph $G$, a valid back-door adjustment set for the effect of $X$ on $Y$ is $\{A, B\}$. After collapsing variable $A$ via the strong d-convex hull, the adjustment set reduces to $\{B\}$ alone. Proposition 1 in \cite{liu2020collapsible} further confirms that  $\{A\}$ is indeed unnecessary for adjustment.
\end{example}

\begin{theorem}{\rm
	If the empty set is a valid back-door adjustment set for $(X,Y)$ in $G_R$, then $G$ is causal estimate collapsible onto $R$ if $R$ is the strong d-convex hull of $X$ and $Y$ in $G_{\bar X}$.} 
\end{theorem}

\begin{proof} Suppose $Z=\emptyset$ is an adjustment set in $G_R$, where $R$ is the strong d-convex hull of $X$ and $Y$ in $G$ obtained by the ISCHA algorithm (Algorithm~\ref{Alg-c}). This implies that no back-door paths exist in $G_R$. Next, we consider  two cases:

Case (i) No back-door path exists in the original DAG $G$. Then, $X$ has no parents in $G$, and $Z=\emptyset$ is also a valid adjustment set in $G$.

Case (ii) A back-door path exists in $G$. Let $U$ be a parent of $X$ on this path. Since $Z = \emptyset$, $U$ is omitted from $Z$.  Because the ISCHA algorithm absorbs the shortest inducing paths, $U$ is bypassed by another node $V \in R$ and is therefore omitted. Thus, the local topological structure must satisfy $U \rightarrow X$ (since $U$ lies on a back-door path) and $X \rightarrow V$ (since $G_R$ contains no back-door paths). To maintain acyclicity, the edge $U \rightarrow V$ must exist. Now consider the manipulated graph $G_{\bar{X}}$. In $G_{\bar{X}}$, node $V$ remains a child of both $X$ and $U$. Since $V \in R$ and both $X$ and $U$ are parents of $V$, by the linear ordering condition, the parent $U$ must also be absorbed into the strong d-convex hull of $X$ and $Y$ in $G_{\bar X}$. Once $U$ is absorbed, the back-door path is captured within $G_R$, and the adjustment set becomes non-empty. The result, therefore, reduces to Theorem~\ref{th:validity}.
\end{proof}

In causal effect estimation, the assumption of a known causal network is very strong. In fact, one can currently obtain the Markov equivalence class underlying the observational data only via structure learning algorithms, and this class is often characterized by a CPDAG. To enhance collapsibility for causal estimates, we next propose a theory of collapsibility for causal estimation, specifically for CPDAGs.

\begin{theorem}\label{prop:CPDAG}{\rm
Suppose that $G$ is a DAG that is Markov equivalent to the CPDAG $\mathcal{G}$. For any two distinct vertices $X$ and $Y \notin \operatorname{pa}_{\mathcal{G}}(X)$, 
	a vertex set $R \supseteq \{X, Y\}$ is causal estimate collapsible in $\mathcal{G}$ if and only if it is causal estimate collapsible in $G$.}
\end{theorem}
\begin{proof}
Since all DAGs in the Markov equivalence class represented by $\mathcal{G}$ share the same conditional independence relations, they also have identical collections of $mf$-pairs. According to the definition of collapsible sets, for any two distinct DAGs $G_1$ and $G_2$ in the equivalence class, we have $\mathcal{MF}_{G_1}(R) = \mathcal{MF}_{G_2}(R)$. By Lemma~\ref{lem-mf} and Lemma~\ref{pro-convex hull}, we know the core of ISCHA is absorbing the $mf$-pairs, so they produce the same strong d-convex hull in $G_1$ and $G_2$. Consequently, because $\mathcal{G}$ represents this equivalence class, the strong d-convex hull of $\mathcal{G}$ is identical to that of any of its consistent DAGs.
\end{proof}

The set $R$ is purely a collection of vertices; the induced subgraphs $G_R$ in $\mathcal{G}_R$ may differ structurally. Nevertheless, (causal) estimate collapsibility of $R$ is invariant across all Markov equivalent DAGs in $\mathcal{G}_R$.

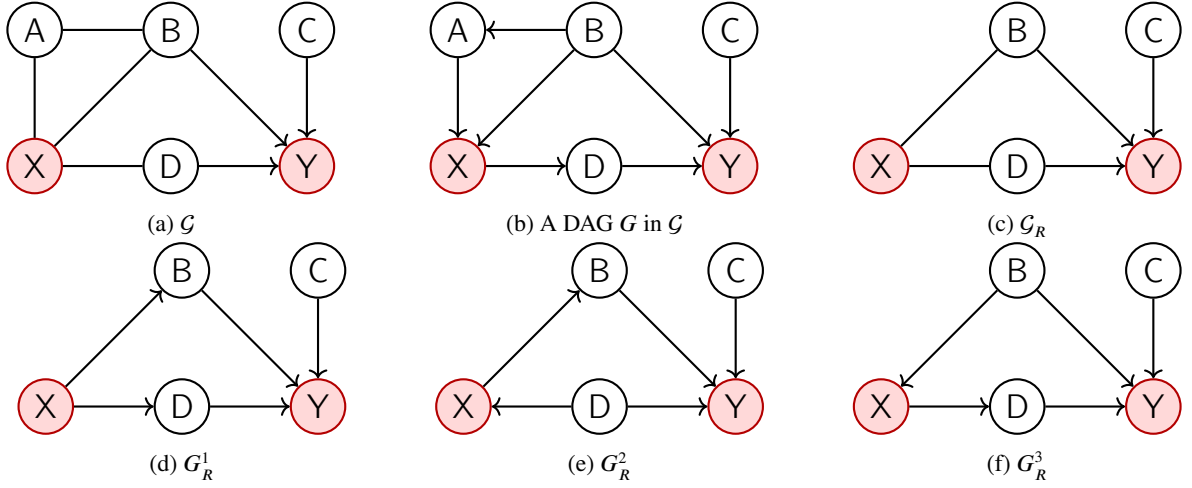
\begin{figure}[pos=htbp]
	\centering
	\begin{subfigure}[t]{0.32\textwidth}
		\centering
		\begin{tikzpicture}[
			scale=0.9,
			transform shape,
			nodeR/.style = {fill=red!15,draw=red!70!black}, 
			nodeMB/.style = {fill=red!15,draw=black, thick}, 
			every node/.style={
				draw, 
				circle, 
				thick, 
				minimum size=8mm,
				inner sep=0pt, 
				align=center,
				font = \Large
			}]
			\node [nodeR](X) at (0,0) {X};
			\node (D) at (2,0) {D};
			\node [nodeR](Y) at (4,0) {Y};
			\node (A) at (0,2) {A};
			\node (B) at (2,2) {B};
			\node (C) at (4,2) {C};
			\draw [-, thick] (X) -- (D);  
			\draw [->, thick] (D) -- (Y);  
			\draw [-, thick] (A) -- (X);  
			\draw [-, thick] (A) -- (B);  
			\draw [-, thick] (B) -- (X);  
			\draw [->, thick] (B) -- (Y);  
			\draw [->, thick] (C) -- (Y);    
			
		\end{tikzpicture}
		\caption{$\mathcal{G}$}
		\label{fig:CPDAG-example1a}
	\end{subfigure}
	\hfill
	\begin{subfigure}[t]{0.32\textwidth}
		\centering
		\begin{tikzpicture}[
			scale=0.9,
			transform shape,
			nodeR/.style = {fill=red!15,draw=red!70!black}, 
			nodeMB/.style = {fill=red!15,draw=black, thick}, 
			every node/.style={
				draw, 
				circle, 
				thick, 
				minimum size=8mm,
				inner sep=0pt, 
				align=center,
				font = \Large
			}]
			
			\node [nodeR](X) at (0,0) {X};
			\node (D) at (2,0) {D};
			\node [nodeR](Y) at (4,0) {Y};
			\node (A) at (0,2) {A};
			\node (B) at (2,2) {B};
			\node (C) at (4,2) {C};
			\draw [->, thick] (X) -- (D);  
			\draw [->, thick] (D) -- (Y);  
			\draw [->, thick] (A) -- (X);  
			\draw [->, thick] (B) -- (A);  
			\draw [->, thick] (B) -- (X);  
			\draw [->, thick] (B) -- (Y);  
			\draw [->, thick] (C) -- (Y);     
			
		\end{tikzpicture}
		\caption{A DAG $G$ in $\mathcal{G}$}
		\label{fig:CPDAG-example1b}
	\end{subfigure}
	\hfill
	\begin{subfigure}[t]{0.32\textwidth}
		\centering
		\begin{tikzpicture}[
			scale=0.9,
			transform shape,
			nodeR/.style = {fill=red!15, draw=red!70!black}, 
			nodeMB/.style = {fill=red!15,draw=black, thick},
			nodeP2/.style = {fill=blue!15, draw=blue!70!black, dotted},
			nodePath/.style = {fill=green!15, draw=green!70!black, dotted},
			nodep/.style = {fill=yellow!20, draw=yellow!80!black, dashed},
			every node/.style={
				draw, 
				circle, 
				thick, 
				minimum size=8mm,
				inner sep=0pt, 
				align=center,
				font = \Large
			}]		
			\node [nodeR](X) at (0,0) {X};
			\node (D) at (2,0) {D};
			\node [nodeR](Y) at (4,0) {Y};
			\node (B) at (2,2) {B};
			\node (C) at (4,2) {C};
			\draw [-, thick] (X) -- (D);  
			\draw [->, thick] (D) -- (Y);   
			\draw [-, thick] (B) -- (X);  
			\draw [->, thick] (B) -- (Y);  
			\draw [->, thick] (C) -- (Y);    
		\end{tikzpicture}
		\caption{$\mathcal{G}_R$}
		\label{fig:CPDAG-example1c}
	\end{subfigure}
	
	\hfill
	\begin{subfigure}[t]{0.32\textwidth}
		\centering
		\begin{tikzpicture}[
			scale=0.9,
			transform shape,
			nodeR/.style = {fill=red!15, draw=red!70!black}, 
			nodeMB/.style = {fill=red!15,draw=black, thick},
			nodeP2/.style = {fill=blue!15, draw=blue!70!black, dotted},
			nodePath/.style = {fill=green!15, draw=green!70!black, dotted},
			nodep/.style = {fill=yellow!20, draw=yellow!80!black, dashed},
			every node/.style={
				draw, 
				circle, 
				thick, 
				minimum size=8mm,
				inner sep=0pt, 
				align=center,
				font = \Large
			}]		
			\node [nodeR](X) at (0,0) {X};
			\node (D) at (2,0) {D};
			\node [nodeR](Y) at (4,0) {Y};
			\node (B) at (2,2) {B};
			\node (C) at (4,2) {C};
			\draw [->, thick] (X) -- (D);  
			\draw [->, thick] (D) -- (Y);   
			\draw [->, thick] (X) -- (B);  
			\draw [->, thick] (B) -- (Y);  
			\draw [->, thick] (C) -- (Y);   
		\end{tikzpicture}
		\caption{${G}_R^1$}
		\label{fig:CPDAG-example1d}
	\end{subfigure}
	\hfill
	\begin{subfigure}[t]{0.32\textwidth}
		\centering
		\begin{tikzpicture}[
			scale=0.9,
			transform shape,
			nodeR/.style = {fill=red!15,draw=red!70!black}, 
			nodeMB/.style = {fill=red!15,draw=black, thick}, 
			nodePath/.style = {fill=green!15, draw=green!70!black, dotted},
			nodeP2/.style = {fill=blue!15, draw=blue!70!black, dotted},
			nodep/.style = {fill=yellow!20, draw=yellow!80!black, dashed},
			every node/.style={
				draw, 
				circle, 
				thick, 
				minimum size=8mm,
				inner sep=0pt, 
				align=center,
				font = \Large
			}]
			\node [nodeR](X) at (0,0) {X};
			\node (D) at (2,0) {D};
			\node [nodeR](Y) at (4,0) {Y};
			\node (B) at (2,2) {B};
			\node (C) at (4,2) {C};
			\draw [->, thick] (D) -- (X);  
			\draw [->, thick] (D) -- (Y);   
			\draw [->, thick] (X) -- (B);  
			\draw [->, thick] (B) -- (Y);  
			\draw [->, thick] (C) -- (Y);   
		\end{tikzpicture}
		\caption{${G}_R^2$}
		\label{fig:CPDAG-example1e}
	\end{subfigure}
	\hfill
	\begin{subfigure}[t]{0.32 \textwidth}
		\centering
		\begin{tikzpicture}[
			scale=0.9,
			transform shape,
			nodeR/.style = {fill=red!15, draw=red!70!black}, 
			nodeMB/.style = {fill=red!15,draw=black, thick},
			nodeP2/.style = {fill=blue!15, draw=blue!70!black, dotted},
			nodePath/.style = {fill=green!15, draw=green!70!black, dotted},
			nodep/.style = {fill=yellow!20, draw=yellow!80!black, dashed},
			every node/.style={
				draw, 
				circle, 
				thick, 
				minimum size=8mm,
				inner sep=0pt, 
				align=center,
				font = \Large
			}]		
			\node [nodeR](X) at (0,0) {X};
			\node (D) at (2,0) {D};
			\node [nodeR](Y) at (4,0) {Y};
			\node (B) at (2,2) {B};
			\node (C) at (4,2) {C};
			\draw [->, thick] (X) -- (D);  
			\draw [->, thick] (D) -- (Y);   
			\draw [->, thick] (B) -- (X);  
			\draw [->, thick] (B) -- (Y);  
			\draw [->, thick] (C) -- (Y);   
		\end{tikzpicture}
		\caption{${G}_R^3$}
		\label{fig:CPDAG-example1f}
	\end{subfigure}
	
	\caption{A causal estimate collapsible Set in $\mathcal{G}$.}
	\label{fig:CPDAG-example}
\end{figure}
\begin{example}\label{ex:dagcollap}
Figure~\ref{fig:CPDAG-example}(a) presents a CPDAG $\mathcal{G}$ corresponding to  Figure~\ref{fig:subexample}(a), which has 8 Markov equivalent DAGs. By selecting a Markov equivalent DAG of $\mathcal{G}$  as shown in Figure~\ref{fig:CPDAG-example}(b), the node $\{A\}$ can be collapsed in ${G}$. Applying Theorem~\ref{prop:CPDAG} yields Figure~\ref{fig:CPDAG-example}(c), which contains only 3 Markov equivalent DAGs. These DAGs are listed in Figures~\ref{fig:CPDAG-example}(d)-\ref{fig:CPDAG-example}(f). Thus, applying our method to these three DAGs suffices to obtain all possible causal effects of $X$ on $Y$.
\end{example}

Example~\ref{ex:dagcollap} demonstrates that causal estimate collapsibility can effectively reduce both 
the number of Markov equivalent DAGs to be enumerated and the size of the CPDAG $\mathcal{G}$ when estimating all possible causal effects, thereby significantly improving computational efficiency. Based on these results, we integrate the graph reduction procedure with the IDA framework, as summarized in Algorithm~\ref{alg:collapsible_ida}.

\begin{breakablealgorithm}
	\caption{Subgraph IDA for CPDAGs}\label{alg:collapsible_ida}
	\begin{algorithmic}[1]
		\Require A CPDAG $\mathcal{G}$, treatment $X$, outcome $Y$, and observational data
		\Ensure A multiset $\Theta$ of possible causal effects of $X$ on $Y$
		\State Initialize $\Theta \leftarrow \emptyset$
		
		\State Obtain an arbitrary DAG $G \in [\mathcal{G}]$ via Meek's rules~\citep{meek1995}
		\State $R_{\text{init}} \leftarrow \mathrm{ISCHA}(G, \{X,Y\})$ 
		\For{each DAG $G \in [\mathcal{G}]$ consistent with $\mathcal{G}$}
		\State $R \leftarrow R_{\text{init}}$
		\State $Z \leftarrow \operatorname{pa}_{G_R}(X)$
		
		\If{$Z = \emptyset$}
		\State Construct the manipulated graph $G_{\bar X}$
		\State $R \leftarrow \mathrm{ISCHA}(G_{\bar X}, \{X,Y\})$
		\State $Z \leftarrow \operatorname{pa}_{G_R}(X)$
		\EndIf
		\If{$Z = \emptyset$}
		\State Estimate $\theta$ via $P(Y\mid X)$
		\Else
		\State Estimate $\theta$ via $\sum_z P(Y\mid X,Z=z)P(Z=z)$
		\EndIf
		\State $\Theta \leftarrow \Theta \cup \{\theta\}$
		\EndFor
		\State \Return $\Theta$
	\end{algorithmic}
\end{breakablealgorithm}

To detail the performance of our proposed algorithms for causal effect estimation, we first employ various simulations to demonstrate their benefits.

\section{ Numerical Simulation and Empirical Analysis}\label{sec5}
Based on this collapsible strategy and the theoretical results above,  we design four experiments to evaluate the performance of our method. 

(i) Demonstrate how the proposed collapsible strategy reduces the model dimension.

(ii) For probabilistic reasoning, the ISCHA algorithm is used to extract a minimal collapsible subgraph containing the target variables, compute the marginal probability on this subgraph, verify its consistency with the full-network inference, and assess the efficiency gain.
 
 (iii) For causal effect estimation, we apply Subgraph IDA on the local subgraph, compare its results with those from IDA, and verify both their consistency and the computational speedup.
 
 (iv) We evaluate the scalability and practical utility of Subgraph IDA on CBNs of varying sizes and sample sizes.

\subsection{Collapsible Strategy }
  
Figure~\ref{fig:hailfinder}(a) presents a weather forecasting CBN consisting of 56 variables from the bnlearn repository \footnote{\url{https://www.bnlearn.com/bnrepository/}}. We focus on two key variables: the forecast for the mountainous region (MountainFcst) and the area of moisture and dry air (AreaMoDryAir). Inferences on these two variables can be conducted using a much smaller CBN, namely the subgraph shown in Figure~\ref{fig:hailfinder}(b), which is induced by our ISCHA algorithm and consists of 16 variables. Thus, the original 56-dimensional model is collapsed onto a 16-dimensional submodel, thereby reducing computational complexity for subsequent research. The detailed reduction process, described in Algorithms~\ref{Alg-ip}--\ref{Alg-c}, is as follows.

Consider the target set $R = \{\text{AreaMoDryAir, MountainFcst}\}$. First, by invoking Algorithm~\ref{Alg-ip}, we identify the initial $\mathrm{AIP}$ set 
$\{\text{AreaMoDryAir, CldShadeOth, InsInMt, MountainFcst}\}.$
This step captures the minimal inducing paths between the target variables, effectively pruning 52 irrelevant nodes. Subsequently, Algorithm~\ref{Alg-d} verifies the d-convexity of this set. In this instance, the procedure returns $H = \mathrm{AIP}$, confirming that no additional nodes are required to satisfy d-convexity.
However, while these four variables preserve model collapsibility, the requirements for estimate collapsibility are stricter. Starting from the current set $H$, Algorithm~\ref{Alg-c} iteratively evaluates linear ordering as defined in Definition~\ref{def:sc}. It identifies that several key parent nodes such as \text{CombVerMo}, \text{SubjVertMo}, and \text{AreaMeso\_ALS} must be included to satisfy the strong d-convex criterion. After several iterations of linear ordering, the original 56-variable model is collapsed onto a 16-variable submodel that is sufficient for exact inference on $R$, significantly reducing computational complexity, as shown in Figure~\ref{fig:hailfinder}(b).

\begin{figure}[pos=htbp]
	\centering
	\begin{minipage}[b]{0.65\textwidth}
		\centering
		\includegraphics[width=\textwidth]{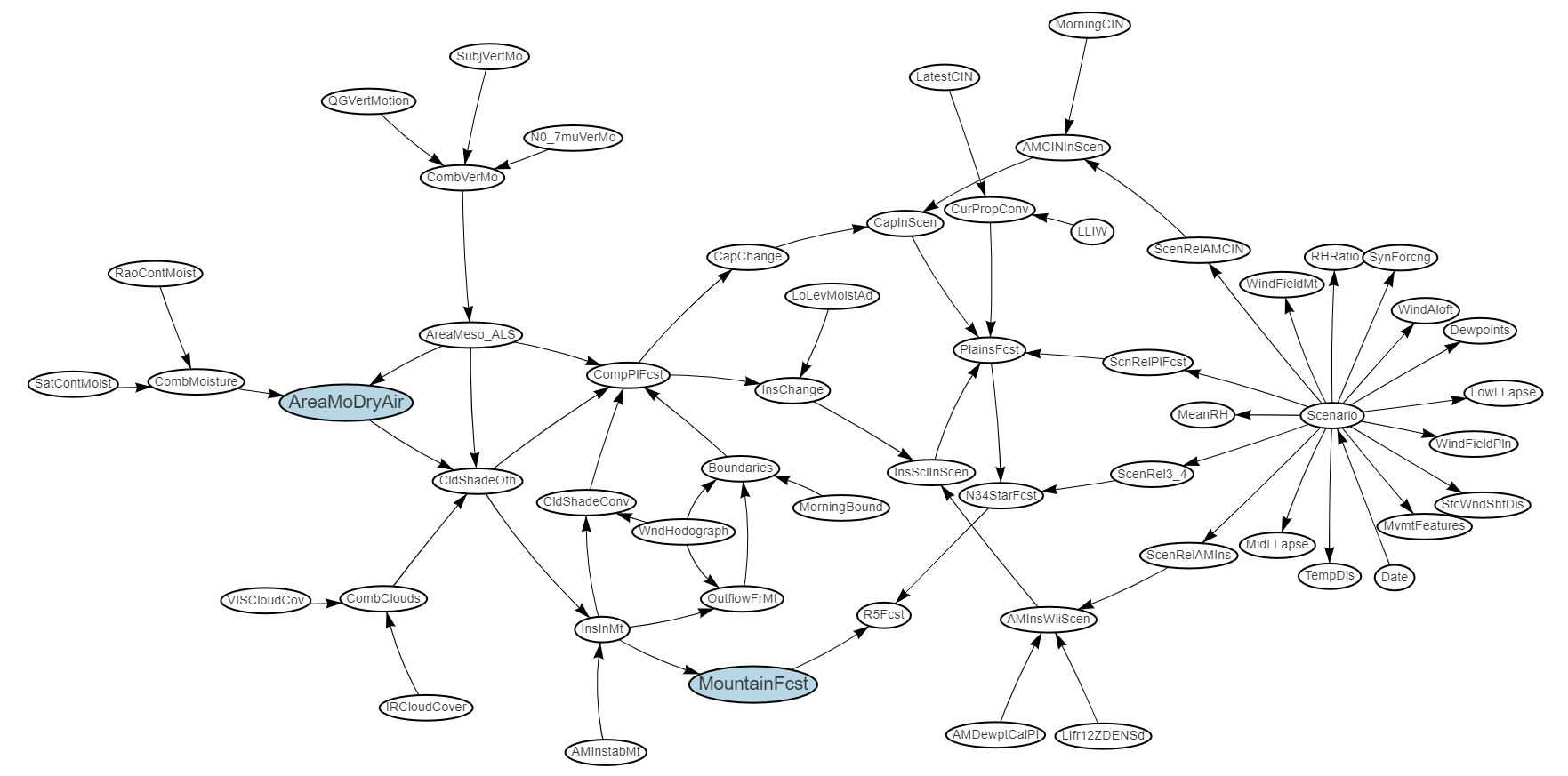}
	\end{minipage}
	\hfill 
	\begin{minipage}[b]{0.30\textwidth}
		\includegraphics[width=\textwidth]{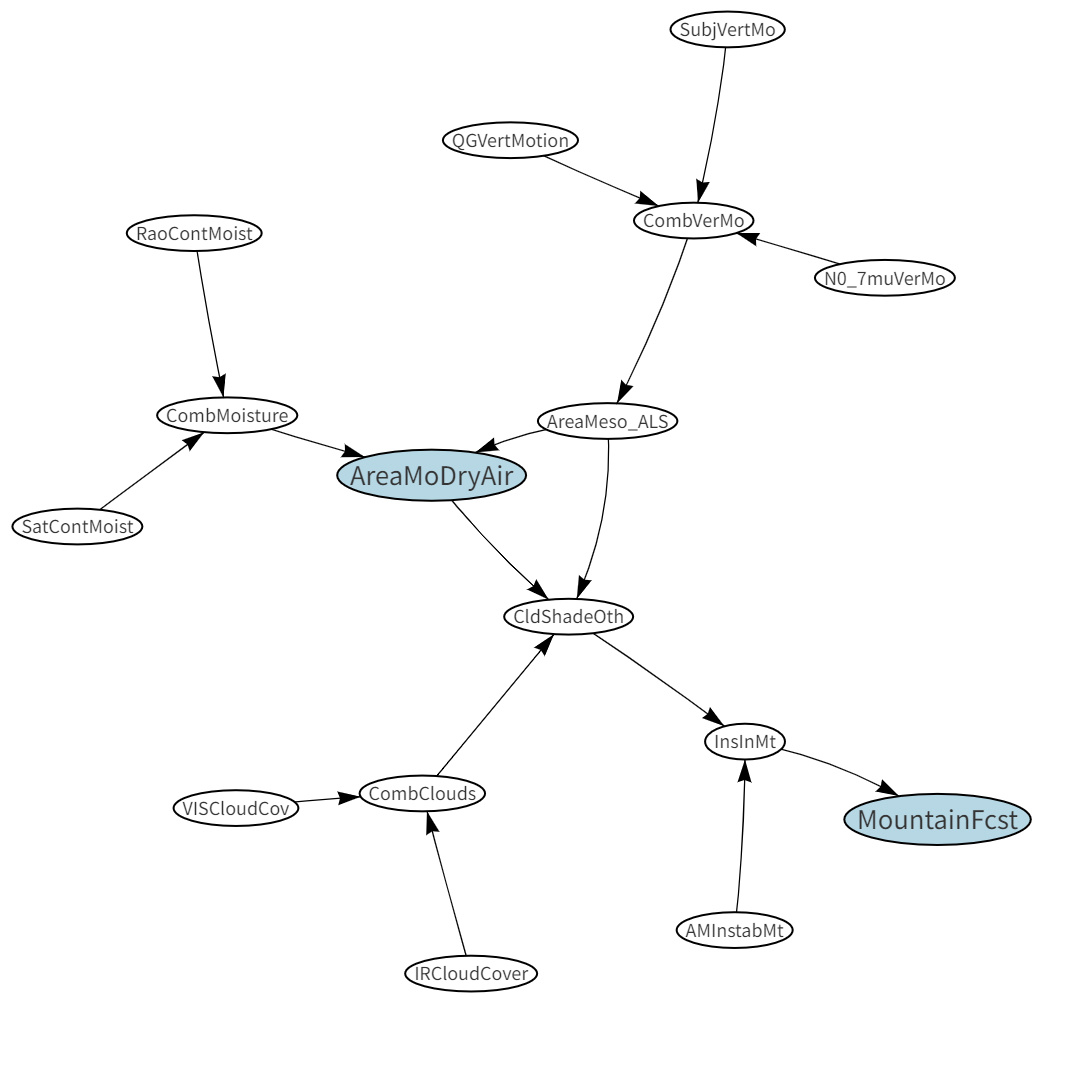}
	\end{minipage}
	\caption{(a) A graph of the Hailfinder network consisting of 56 variables, and (b) a subgraph induced by the minimal collapsible set containing the variables of interest.}
	\label{fig:hailfinder}
\end{figure}

\subsection{Probabilistic Reasoning}
The central idea is to leverage the ISCHA algorithm to answer specific probabilistic queries by extracting a minimal collapsible subgraph, thereby avoiding costly global inference. In this experiment, we select CBNs of varying sizes from the ~\href{https://www.bnlearn.com/bnrepository/}{bnlearn repository}. For each CBN, we randomly select 100 pairs of target variables $(X, Y)$. Simulated observational datasets are generated via forward sampling from the true joint distribution of the network, with sample sizes $N \in \{500, 1000, 2500, 5000, 7500, 10000\}$. For each sample size, the experiment is repeated 100 times. Given the known network structure and sampled data, the model parameters are learned by MLE. We compute the exact marginal joint probability $P(X, Y)$ using the variable elimination algorithm under two schemes: (i) Model fitting and probabilistic reasoning are performed on the entire original network;  (ii) The ISCHA algorithm first extracts the collapsed subgraph for the target pair, then reasoning is conducted on the reduced subgraph. 

We measure the discrepancy between the probability distributions obtained from the global approach $P$ and the local approach $Q$  using the
Kullback-Leibler divergence, defined as $D_{\mathrm{KL}}(P\|Q) = \sum_{x \in \mathcal{X}} P(x) \log \frac{P(x)}{Q(x)}.$
Concurrently, we calculate the speedup ratio ($\mathrm{SR}$) and the node reduction ratio ($\mathrm{NRR}$), defined respectively as  
$$\mathrm{SR} = \frac{\mathrm{Time}(\mathrm{Full}) - \mathrm{Time}(\mathrm{Local})}{\mathrm{Time}(\mathrm{Full})},
 \qquad
\mathrm{NRR}=\frac{\mathrm{Node}(\mathrm{Full}) - \mathrm{Node}(\mathrm{Local})}{\mathrm{Node}(\mathrm{Full})}.
$$

\begin{table}[pos=htbp]
	\centering
	\caption{Performance of the ISCHA-based Local approach for probabilistic queries.}
	\label{tab}
	\renewcommand{\arraystretch}{1.3}
	\setlength{\tabcolsep}{3.5mm}
\begin{tabular*}{\linewidth}{@{\extracolsep{\fill}}lrrrrrrr@{}}
		\toprule
		\multirow{2}{*}{Networks}
		& \multicolumn{2}{c}{Nodes}
		& \multicolumn{2}{c}{Time (s)}
		& KL-divergence
		& NRR
		& SR \\
		\cmidrule(lr){2-3} \cmidrule(lr){4-5}
		& Full  & Local & Full  & Local & (Full vs Local) &  (\%) & (\%) \\
		\midrule
		Hailfinder   & 56    & 14.19 & 0.00107 & 0.00104 & $1.19 \times 10^{-17}$  & 75 & 3 \\
		Win95pts     & 76    & 8.77  & 0.00517 & 0.00469 & $8.64 \times 10^{-18}$  & 88 & 9 \\
		Pathfinder   & 109   & 4.68  & 0.00051 & 0.00040 & $4.99 \times 10^{-18}$  & 96 & 22 \\
		Munin1       & 186   & 22.35 & 0.00207 & 0.00193 & $9.14 \times 10^{-19}$ & 88 & 7 \\
		Andes        & 223   & 68.83 & 0.00517 & 0.00469 & $8.64 \times 10^{-18}$  & 69 & 9 \\
		\bottomrule
	\end{tabular*}
\end{table}
Table~\ref{tab} demonstrates that the local subgraph accurately preserves the probability distribution of the full graph, with the KL divergence being negligible (on the order of $10^{-18}$).
This high accuracy is accompanied by a substantial reduction in graph scale, where the node reduction ratio reaches up to 96\%.
Moreover, the speedup ratio shows that our local inference method is computationally efficient, and the query-time savings clearly exceed the overhead of graph collapse, thereby outperforming global inference in total runtime.
 
\subsection{ Causal Effect Estimation}
To evaluate the accuracy and computational efficiency of the proposed Subgraph IDA algorithm in estimating causal effects, we generate underlying true DAGs based on Erdős-Rényi random graphs. We vary the number of vertices $|V| \in \{20,40,60,80,100\}$, and the edge density $d \in \{0.05,0.1,0.15\}$, and then derive their corresponding CPDAGs. For each configuration, we randomly select 100 node pairs to serve as the treatment and outcome variables. We compute the causal effects using both the proposed Subgraph IDA and the IDA (implemented via the \texttt{ida()} function in the \texttt{pcalg} package)\citep{maathuis2010}.  

We employ the following indices to comprehensively evaluate the performance.
{(i) Recall and Precision.} For a given equivalence class represented by a CPDAG, the IDA algorithm returns a multiset of possible causal effects. To systematically compare the sets of causal effects estimated on the full graph versus the local subgraph, we define the following metrics:
$$\mathrm{REC}=\frac{|E_{\mathrm{full}}\cap E_{\mathrm{local}}|}{|E_{\mathrm{full}}|} , \mathrm{PREC} = \frac{|E_{\mathrm{local}}\cap E_{\mathrm{full}}|}{|E_{\mathrm{local}}|},$$
where $E_{\mathrm{full}}$ and $E_{\mathrm{local}}$ denote the sets of causal effect values (after removing duplicates) estimated on the full graph and the local subgraph, respectively. Achieving $\mathrm{REC} = 1$ and $\mathrm{PREC} = 1$ demonstrates that the local estimation perfectly matches the global estimation, empirically confirming that our approach neither loses any genuine causal information nor introduces any spurious causal effects.
{(ii) Speedup Factor.} To directly illustrate the computational efficiency advantage of Subgraph IDA, we calculate the speedup factor:
$\mathrm{S} = \frac{T_{\mathrm{full}}}{T_{\mathrm{local}}},$
where $T_{\mathrm{full}}$ denotes the execution time for estimating causal effects using IDA, and $T_{\mathrm{local}}$ represents the execution time of our Subgraph IDA. A speedup factor $S \ge 1$ indicates that the Subgraph IDA algorithm achieves faster computational performance.

\begin{figure}[pos=htbp]
	\centering
	\includegraphics[width=0.96\textwidth]{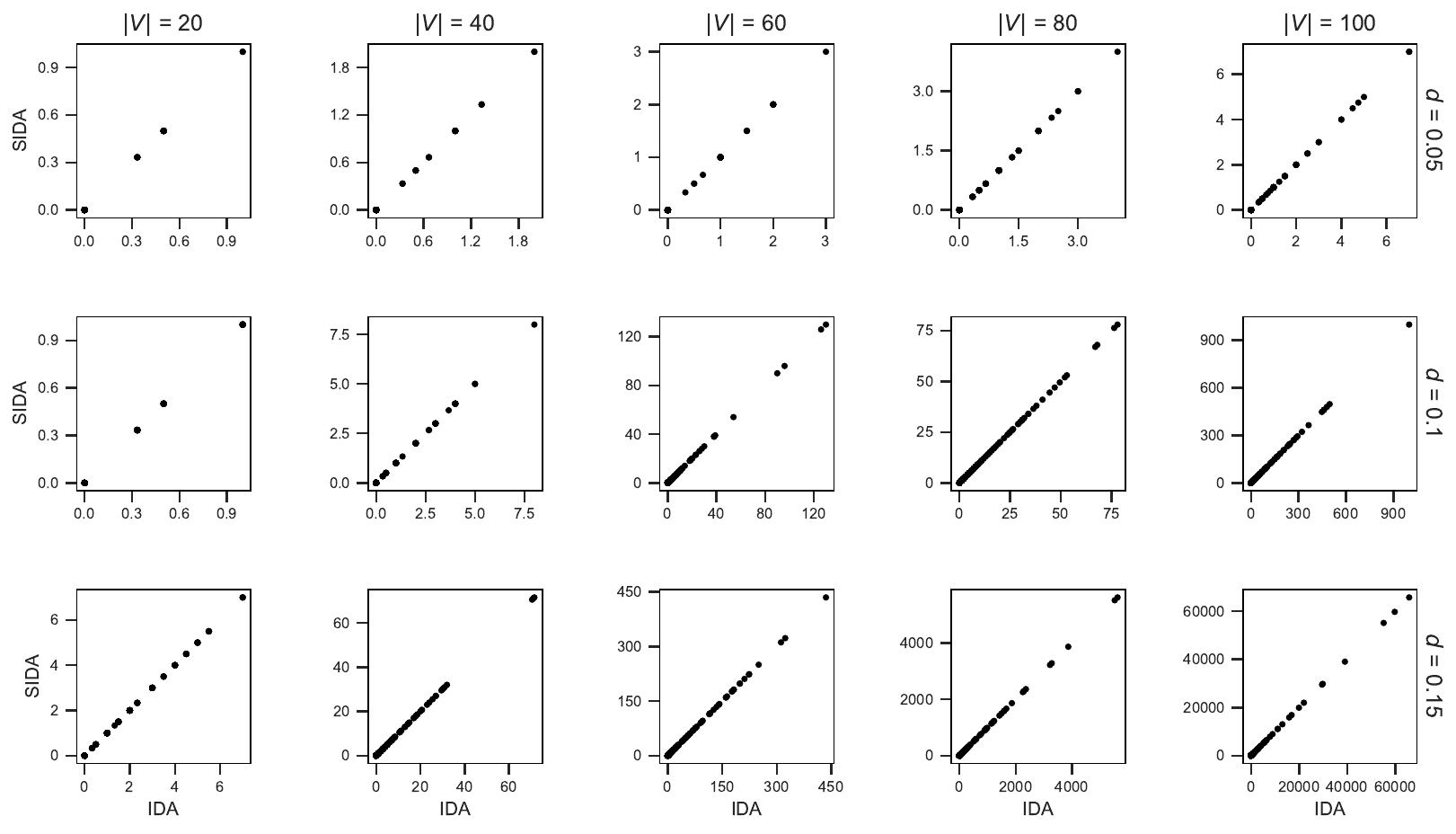}
	\caption{Causal Effect Performance of Subgraph IDA (SIDA) vs. IDA under Varying Node Sizes and Network Densities.}
		\label{fig:mainresults1}
\end{figure}
\begin{table}[pos=htbp]
	\centering
	\caption{Experimental Results on Random Graphs}
	\label{tab2}
\begin{tabular*}{\linewidth}{@{\extracolsep{\fill}}ccccccc@{}}
		\toprule
		\multirow{2}{*}{Nodes} & \multirow{2}{*}{Density} & \multirow{2}{*}{Recall} & \multirow{2}{*}{Precision} 
		& \multicolumn{2}{c}{Time (s)} & \multirow{2}{*}{Speedup} \\
		\cmidrule(lr){5-6}
		& & & & Local & Full & \\
		\midrule
		20  & 0.05 & 1.0000 & 1.0000 & 0.0039 & 0.0099 & 2.7973  \\
		20  & 0.10 & 1.0000 & 1.0000 & 0.0041 & 0.0114 & 3.0075  \\
		20  & 0.15 & 1.0000 & 1.0000 & 0.0066 & 0.0100 & 1.9381  \\
		\addlinespace
		40  & 0.05 & 1.0000 & 1.0000 & 0.0068 & 0.0339 & 5.6806  \\
		40  & 0.10 & 1.0000 & 1.0000 & 0.0084 & 0.0275 & 4.3347  \\
		40  & 0.15 & 1.0000 & 1.0000 & 0.0174 & 0.0430 & 4.8261  \\
		\addlinespace
		60  & 0.05 & 1.0000 & 1.0000 & 0.0088 & 0.0451 & 7.4439  \\
		60  & 0.10 & 1.0000 & 1.0000 & 0.0609 & 0.2018 & 7.1802  \\
		60  & 0.15 & 0.9990 & 1.0000 & 0.0659 & 0.1963 & 7.6221  \\
		\addlinespace
		80  & 0.05 & 1.0000 & 1.0000 & 0.0112 & 0.0928 & 13.9504 \\
		80  & 0.10 & 1.0000 & 1.0000 & 0.1020 & 0.2606 & 12.2631 \\
		80  & 0.15 & 1.0000 & 1.0000 & 0.2329 & 0.4982 & 13.5119 \\
		\addlinespace
		100 & 0.05 & 0.9978 & 1.0000 & 0.0261 & 0.1853 & 18.2208 \\
		100 & 0.10 & 1.0000 & 1.0000 & 0.4157 & 1.0523 & 17.4892 \\
		100 & 0.15 & 1.0000 & 1.0000 & 1.0781 & 1.7683 & 17.5248 \\
		\bottomrule
	\end{tabular*}
\end{table}

Figure~\ref{fig:mainresults1} visually validates the consistency of causal effect estimates between Subgraph IDA (SIDA) and the IDA method across different network sizes $(|V| = 20, 40, 60, 80, 100)$ and graph densities $(d = 0.05, 0.1, 0.15)$. Each subplot presents a scatter plot of causal effect values estimated by SIDA against those from IDA, where all points lie perfectly along the diagonal line. This demonstrates that SIDA produces identical causal effect estimates to the original IDA method, regardless of network scale or density. Complementing this visual evidence, Table~\ref{tab2} summarizes the experimental results across different network sizes and graph densities.  As shown, both REC and PREC remain nearly consistently at 1.0 across all experimental configurations. This confirms that the multiset of causal effects estimated by Subgraph IDA is identical to the IDA output, preserving all genuine causal information without introducing spurious effects. Meanwhile, the local method consistently reduces computational time compared with the full-graph approach, yielding substantial speedup across all configurations. The time savings are particularly pronounced for medium to large networks, demonstrating the efficiency advantage of performing causal effect estimation on a collapsed local subgraph rather than the full graph.



\subsection{Empirical Analysis}
To evaluate the scalability and practical utility of Subgraph IDA, we select representative CBNs from the ~\href{https://www.bnlearn.com/bnrepository/}{bnlearn repository}, ranging from small to medium-sized networks (e.g., Sachs, Alarm, Insurance) to large networks (e.g., Hepar II, Pathfinder, Munin1). For each network and each sample size $N$ (ranging from 500 to 10,000), we randomly sampled multiple pairs of treatment and outcome variables and repeated the experiment 50 times. We estimate causal effects using noisy covariance matrices to simulate finite sample errors, allowing us to compare the exactness and speedup of Subgraph IDA against IDA.

\begin{figure}[pos=htbp]
	\centering
	\centering
	\includegraphics[width=\textwidth]{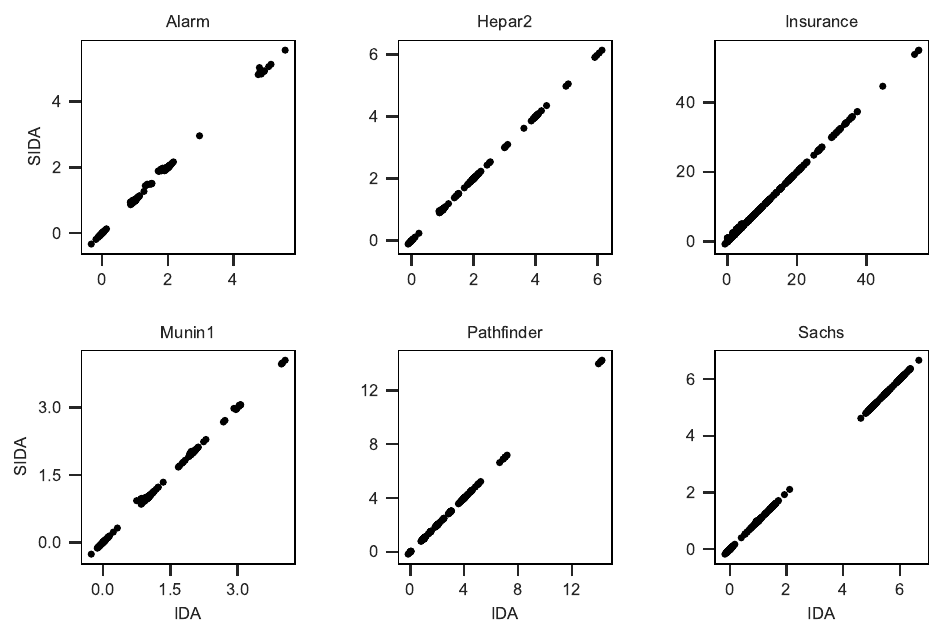}
	\caption{Causal Effect Performance of Subgraph IDA (SIDA) vs. IDA under Varying Networks.}
	\label{fig:main_results2}
\end{figure}
\begin{table}[pos=htbp]
	\centering
	\caption{Experimental Results on Real-World BNs}
	\label{tab3}
	\begin{tabular*}{\linewidth}{@{\extracolsep{\fill}}ccccccc@{}}
		\toprule
		\multirow{2}{*}{Network} 
		& \multirow{2}{*}{Recall} 
		& \multirow{2}{*}{Precision} 
		& \multicolumn{2}{c}{Time (s)} 
		& \multirow{2}{*}{Speedup} \\
		\cmidrule(lr){4-5}
		& & & Local & Full & \\
		\midrule
		Sachs       & 0.9367  & 1.0000  & 0.0088  & 0.0095  & 1.13    \\
		Insurance   & 0.9183  & 1.0000  & 0.0215  & 0.0290  & 1.82    \\
		Alarm       & 0.8633  & 1.0000  & 0.0118  & 0.0354  & 3.44    \\
		Hepar2      & 0.9678  & 1.0000  & 0.0856  & 0.1788  & 3.00    \\
		Pathfinder  & 0.9933  & 1.0000  & 0.0282  & 0.1337  & 7.27    \\
		Munin1      & 0.9050  & 1.0000  & 0.1342  & 0.8026  & 30.37   \\
		\bottomrule
	\end{tabular*}
\end{table}

 Figure~\ref{fig:main_results2} and Table~\ref{tab3} present the comparison of causal effect estimates on real world BNs. The scatter plots show that the two methods yield nearly identical causal effect values across all tested networks, with points tightly distributed along the diagonal line. As shown in the table, the precision remains 1.0 across all networks, while the recall remains relatively high ($\ge$ 0.86). The perfect precision demonstrates that Subgraph IDA does not introduce any spurious causal effects. The slight drop in recall can be attributed to finite sample noise: when true causal effects are zero, IDA may produce near-zero estimates due to sampling variability, while Subgraph IDA tends to exclude such negligible values after pruning irrelevant graph structures, leading to a minor reduction in recall. The speedup factor $S$ increases with network size ($S \geq 1$ across all networks), ranging from a mild improvement on small networks to up to 30.37$\times$ on large networks such as Munin1, confirming that Subgraph IDA becomes more efficient as networks scale up.
%

\section{ Conclusion}\label{sec6}
 {To efficiently estimate causal effects for CPDAGs in high-dimensional settings, we introduced the concept of estimate collapsibility for causal effect estimation. We utilized strong d-convex hulls to characterize minimal collapsible sets and designed an algorithm to identify them. Based on the graph reduction procedure, we combined it with the back-door criterion to estimate causal effects for CPDAGs, which reduces the enumeration of Markov equivalent DAGs and adjustment sets. Our method ensures that the causal effect estimated from the local submodel coincides exactly with that from the full model, while delivering substantial speedup in high dimensional settings. These key properties make our method highly practical for real world applications.
 	
Nonetheless, our work has certain limitations that merit attention. Specifically, the current definition of inducing paths restricts causal effect estimation to non-adjacent $(X,Y)$ pairs, i.e., those with purely indirect causal effects. Moreover, the framework does not yet account for graphs with latent variables, and we have not explored whether other types of adjustment sets can benefit from collapsibility. These limitations provide promising directions for future research.}









\section*{Acknowledgement}
The first author and the third author were supported by the National Natural Science Foundation of China (12561047), the Xinjiang Talent Development Fund (XJRC-2025-KJ-PY-KJLJ-108), and the 2025 Central Guidance for Local Science and Technology Development Fund (ZYYD2025ZY20). 
The second author was supported in part by the National Natural Science Foundation of China (12426520 and 12426105).

\printcredits

\bibliographystyle{cas-model2-names}

\bibliography{cas-refs}

\begin{thebibliography}{41}
\expandafter\ifx\csname natexlab\endcsname\relax\def\natexlab#1{#1}\fi
\providecommand{\url}[1]{\texttt{#1}}
\providecommand{\href}[2]{#2}
\providecommand{\path}[1]{#1}
\providecommand{\DOIprefix}{doi:}
\providecommand{\ArXivprefix}{arXiv:}
\providecommand{\URLprefix}{URL: }
\providecommand{\Pubmedprefix}{pmid:}
\providecommand{\doi}[1]{\href{http://dx.doi.org/#1}{\path{#1}}}
\providecommand{\Pubmed}[1]{\href{pmid:#1}{\path{#1}}}
\providecommand{\bibinfo}[2]{#2}
\ifx\xfnm\relax \def\xfnm[#1]{\unskip,\space#1}\fi
\bibitem[{Asmussen and Edwards(1983)}]{asmussen1983}
\bibinfo{author}{Asmussen, S.}, \bibinfo{author}{Edwards, D.},
  \bibinfo{year}{1983}.
\newblock \bibinfo{title}{Collapsibility and response variables in contingency
  tables}.
\newblock \bibinfo{journal}{Biometrika} \bibinfo{volume}{70},
  \bibinfo{pages}{567--578}.
\bibitem[{Braun and Schwartz(2025)}]{braun2025}
\bibinfo{author}{Braun, M.}, \bibinfo{author}{Schwartz, E.M.},
  \bibinfo{year}{2025}.
\newblock \bibinfo{title}{Where {A}/{B} testing goes wrong: How divergent
  delivery affects what online experiments cannot (and can) tell you about how
  customers respond to advertising}.
\newblock \bibinfo{journal}{Journal of Marketing} \bibinfo{volume}{89},
  \bibinfo{pages}{71--95}.
\bibitem[{Chickering(2002)}]{chickering2002learning}
\bibinfo{author}{Chickering, D.M.}, \bibinfo{year}{2002}.
\newblock \bibinfo{title}{Learning equivalence classes of {B}ayesian network
  structures}.
\newblock \bibinfo{journal}{Journal of Machine Learning Research}
  \bibinfo{volume}{2}, \bibinfo{pages}{445--498}.
\bibitem[{Deng et~al.(2025)Deng, Sun and Liu}]{deng}
\bibinfo{author}{Deng, Y.X.}, \bibinfo{author}{Sun, Y.}, \bibinfo{author}{Liu,
  H.X.}, \bibinfo{year}{2025}.
\newblock \bibinfo{title}{Identifying collapsible sets in directed graphical
  models via inducing paths}.
\newblock \bibinfo{journal}{Statistics and Computing} \bibinfo{volume}{35},
  \bibinfo{pages}{115}.
\bibitem[{Didelez and Edwards(2004)}]{didelez2004}
\bibinfo{author}{Didelez, V.}, \bibinfo{author}{Edwards, D.},
  \bibinfo{year}{2004}.
\newblock \bibinfo{title}{Collapsibility of graphical {CG}-regression models}.
\newblock \bibinfo{journal}{Scandinavian Journal of Statistics}
  \bibinfo{volume}{31}, \bibinfo{pages}{535--551}.
\bibitem[{Edwards(1990)}]{edwards1990}
\bibinfo{author}{Edwards, D.}, \bibinfo{year}{1990}.
\newblock \bibinfo{title}{Hierarchical interaction models}.
\newblock \bibinfo{journal}{Journal of the Royal Statistical Society: Series B
  (Methodological)} \bibinfo{volume}{52}, \bibinfo{pages}{3--20}.
\bibitem[{Guo and Perkovic(2021)}]{guo2021minimal}
\bibinfo{author}{Guo, R.}, \bibinfo{author}{Perkovic, E.},
  \bibinfo{year}{2021}.
\newblock \bibinfo{title}{Minimal enumeration of all possible total effects in
  a {M}arkov equivalence class}, in: \bibinfo{booktitle}{International
  Conference on Artificial Intelligence and Statistics},
  \bibinfo{organization}{PMLR}. pp. \bibinfo{pages}{2395--2403}.
\bibitem[{Hauser and B\"{u}hlmann(2012)}]{hauser2012}
\bibinfo{author}{Hauser, A.}, \bibinfo{author}{B\"{u}hlmann, P.},
  \bibinfo{year}{2012}.
\newblock \bibinfo{title}{Characterization and greedy learning of
  interventional {M}arkov equivalence classes of directed acyclic graphs}.
\newblock \bibinfo{journal}{Journal of Machine Learning Research}
  \bibinfo{volume}{13}, \bibinfo{pages}{2409--2464}.
\bibitem[{Henckel et~al.(2024)Henckel, Buttenschoen and Maathuis}]{Henckel2024}
\bibinfo{author}{Henckel, L.}, \bibinfo{author}{Buttenschoen, M.},
  \bibinfo{author}{Maathuis, M.H.}, \bibinfo{year}{2024}.
\newblock \bibinfo{title}{Graphical tools for selecting conditional
  instrumental sets}.
\newblock \bibinfo{journal}{Biometrika} \bibinfo{volume}{111},
  \bibinfo{pages}{771--788}.
\bibitem[{Henckel et~al.(2022)Henckel, Perkovi{\'c} and
  Maathuis}]{henckel2022graphical}
\bibinfo{author}{Henckel, L.}, \bibinfo{author}{Perkovi{\'c}, E.},
  \bibinfo{author}{Maathuis, M.H.}, \bibinfo{year}{2022}.
\newblock \bibinfo{title}{Graphical criteria for efficient total effect
  estimation via adjustment in causal linear models}.
\newblock \bibinfo{journal}{Journal of the Royal Statistical Society Series B:
  Statistical Methodology} \bibinfo{volume}{84}, \bibinfo{pages}{579--599}.
\bibitem[{Heng and Sun(2023)}]{heng2023}
\bibinfo{author}{Heng, P.}, \bibinfo{author}{Sun, Y.}, \bibinfo{year}{2023}.
\newblock \bibinfo{title}{Algorithms for convex hull finding in undirected
  graphical models}.
\newblock \bibinfo{journal}{Applied Mathematics and Computation}
  \bibinfo{volume}{445}, \bibinfo{pages}{127852}.
\bibitem[{Heng et~al.(2026)Heng, Sun and Guo}]{heng2024}
\bibinfo{author}{Heng, P.}, \bibinfo{author}{Sun, Y.}, \bibinfo{author}{Guo,
  J.H.}, \bibinfo{year}{2026}.
\newblock \bibinfo{title}{Structural dimension reduction in {B}ayesian
  networks}.
\newblock \DOIprefix\doi{10.48550/arXiv.2601.08236}.
\bibitem[{Imbens and Rubin(2015)}]{imbens2015causal}
\bibinfo{author}{Imbens, G.W.}, \bibinfo{author}{Rubin, D.B.},
  \bibinfo{year}{2015}.
\newblock \bibinfo{title}{Causal Inference in Statistics, Social, and
  Biomedical Sciences}.
\newblock \bibinfo{publisher}{Cambridge university press}.
\bibitem[{Kim and Kim(2006)}]{kim2006}
\bibinfo{author}{Kim, S.H.}, \bibinfo{author}{Kim, S.H.}, \bibinfo{year}{2006}.
\newblock \bibinfo{title}{A note on collapsibility in {DAG} models of
  contingency tables}.
\newblock \bibinfo{journal}{Scandinavian Journal of Statistics}
  \bibinfo{volume}{33}, \bibinfo{pages}{575--590}.
\bibitem[{Koller and Friedman(2009)}]{koller2009}
\bibinfo{author}{Koller, D.}, \bibinfo{author}{Friedman, N.},
  \bibinfo{year}{2009}.
\newblock \bibinfo{title}{Probabilistic Graphical Models: Principles and
  Techniques}.
\newblock \bibinfo{publisher}{MIT press}.
\bibitem[{Lauritzen(1996)}]{lauritzen1996}
\bibinfo{author}{Lauritzen, S.L.}, \bibinfo{year}{1996}.
\newblock \bibinfo{title}{Graphical Models}.
\newblock \bibinfo{publisher}{Clarendon Press}, \bibinfo{address}{Oxford}.
\bibitem[{Liu and Guo(2013)}]{liu2013}
\bibinfo{author}{Liu, B.H.}, \bibinfo{author}{Guo, J.H.}, \bibinfo{year}{2013}.
\newblock \bibinfo{title}{Collapsibility of conditional graphical models}.
\newblock \bibinfo{journal}{Scandinavian Journal of Statistics}
  \bibinfo{volume}{40}, \bibinfo{pages}{191--203}.
\bibitem[{Liu et~al.(2020)Liu, Fang, He and Geng}]{liu2020collapsible}
\bibinfo{author}{Liu, Y.}, \bibinfo{author}{Fang, Z.Y.}, \bibinfo{author}{He,
  Y.B.}, \bibinfo{author}{Geng, Z.}, \bibinfo{year}{2020}.
\newblock \bibinfo{title}{Collapsible {IDA}: Collapsing parental sets for
  locally estimating possible causal effects}, in:
  \bibinfo{booktitle}{Conference on Uncertainty in Artificial Intelligence},
  \bibinfo{organization}{PMLR}. pp. \bibinfo{pages}{290--299}.
\bibitem[{Maathuis et~al.(2010)Maathuis, Colombo, Kalisch and
  Bühlmann}]{maathuis2010}
\bibinfo{author}{Maathuis, M.}, \bibinfo{author}{Colombo, D.},
  \bibinfo{author}{Kalisch, M.}, \bibinfo{author}{Bühlmann, P.},
  \bibinfo{year}{2010}.
\newblock \bibinfo{title}{Predicting causal effects in large-scale systems from
  observational data}.
\newblock \bibinfo{journal}{Nature methods} \bibinfo{volume}{7},
  \bibinfo{pages}{247--8}.
\bibitem[{Maathuis and Colombo(2015)}]{maathuis2015}
\bibinfo{author}{Maathuis, M.H.}, \bibinfo{author}{Colombo, D.},
  \bibinfo{year}{2015}.
\newblock \bibinfo{title}{A generalized back-door criterion}.
\newblock \bibinfo{journal}{The Annals of Statistics} \bibinfo{volume}{43(3)},
  \bibinfo{pages}{1060--1088}.
\bibitem[{Maathuis et~al.(2009)Maathuis, Kalisch and
  B{\"u}hlmann}]{maathuis2009}
\bibinfo{author}{Maathuis, M.H.}, \bibinfo{author}{Kalisch, M.},
  \bibinfo{author}{B{\"u}hlmann, P.}, \bibinfo{year}{2009}.
\newblock \bibinfo{title}{Estimating high-dimensional intervention effects from
  observational data}.
\newblock \bibinfo{journal}{The Annals of Statistics} \bibinfo{volume}{37(6A)},
  \bibinfo{pages}{3133--3164}.
\bibitem[{Madigan and Mosurski(1990)}]{madigan1990}
\bibinfo{author}{Madigan, D.}, \bibinfo{author}{Mosurski, K.},
  \bibinfo{year}{1990}.
\newblock \bibinfo{title}{An extension of the results of {A}smussen and
  {E}dwards on collapsibility in contingency tables}.
\newblock \bibinfo{journal}{Biometrika} \bibinfo{volume}{77},
  \bibinfo{pages}{315--319}.
\bibitem[{Meek(1995)}]{meek1995}
\bibinfo{author}{Meek, C.}, \bibinfo{year}{1995}.
\newblock \bibinfo{title}{Causal inference and causal explanation with
  background knowledge}, in: \bibinfo{booktitle}{Proceedings of the Eleventh
  Conference on Uncertainty in Artificial Intelligence}, pp.
  \bibinfo{pages}{403--410}.
\bibitem[{Pearl(2009)}]{pearl2009}
\bibinfo{author}{Pearl, J.}, \bibinfo{year}{2009}.
\newblock \bibinfo{title}{Causality}.
\newblock \bibinfo{publisher}{Cambridge university press}.
\bibitem[{Perkovic et~al.(2017)Perkovic, Textor, Kalisch and
  Maathuis}]{perkovic2018}
\bibinfo{author}{Perkovic, E.}, \bibinfo{author}{Textor, J.},
  \bibinfo{author}{Kalisch, M.}, \bibinfo{author}{Maathuis, M.H.},
  \bibinfo{year}{2017}.
\newblock \bibinfo{title}{Complete graphical characterization and construction
  of adjustment sets in {M}arkov equivalence classes of ancestral graphs}.
\newblock \bibinfo{journal}{Journal of Machine Learning Research}
  \bibinfo{volume}{18}, \bibinfo{pages}{8132--8193}.
\bibitem[{Pfaltz(1971)}]{Pfaltz1971}
\bibinfo{author}{Pfaltz, J.L.}, \bibinfo{year}{1971}.
\newblock \bibinfo{title}{Convexity in directed graphs}.
\newblock \bibinfo{journal}{Journal of Combinatorial Theory, Series B}
  \bibinfo{volume}{10}, \bibinfo{pages}{143--162}.
\bibitem[{Pham and Nguyen(2026)}]{PHAM2026}
\bibinfo{author}{Pham, T.T.}, \bibinfo{author}{Nguyen, T.T.},
  \bibinfo{year}{2026}.
\newblock \bibinfo{title}{Causal reasoning over ontology-enriched graphs for
  interpretable medical {AI}}.
\newblock \bibinfo{journal}{Knowledge-Based Systems} \bibinfo{volume}{341},
  \bibinfo{pages}{115803}.
\bibitem[{Renero et~al.(2026)Renero, Maestre and Ochoa}]{Renero2026}
\bibinfo{author}{Renero, J.}, \bibinfo{author}{Maestre, R.},
  \bibinfo{author}{Ochoa, I.}, \bibinfo{year}{2026}.
\newblock \bibinfo{title}{Rex: Causal discovery based on machine learning and
  explainability techniques}.
\newblock \bibinfo{journal}{Pattern Recognition} \bibinfo{volume}{172},
  \bibinfo{pages}{112491}.
\bibitem[{Shadi et~al.(2025)Shadi, Mirshekali and Shaker}]{shadi2025}
\bibinfo{author}{Shadi, M.R.}, \bibinfo{author}{Mirshekali, H.},
  \bibinfo{author}{Shaker, H.R.}, \bibinfo{year}{2025}.
\newblock \bibinfo{title}{Explainable artificial intelligence for energy
  systems maintenance: A review on concepts, current techniques, challenges,
  and prospects}.
\newblock \bibinfo{journal}{Renewable \& Sustainable Energy Reviews}
  \bibinfo{volume}{216}, \bibinfo{pages}{115668}.
\bibitem[{Simpson(1951)}]{simpson1951}
\bibinfo{author}{Simpson, E.H.}, \bibinfo{year}{1951}.
\newblock \bibinfo{title}{The interpretation of interaction in contingency
  tables}.
\newblock \bibinfo{journal}{Journal of the Royal Statistical Society: Series B
  (Methodological)} \bibinfo{volume}{13}, \bibinfo{pages}{238--241}.
\bibitem[{Spirtes et~al.(2001)Spirtes, Glymour and
  Scheines}]{spirtes2001causation}
\bibinfo{author}{Spirtes, P.}, \bibinfo{author}{Glymour, C.},
  \bibinfo{author}{Scheines, R.}, \bibinfo{year}{2001}.
\newblock \bibinfo{title}{Causation, Prediction, and Search}.
\newblock \bibinfo{publisher}{MIT press}.
\bibitem[{Van Der~Zander et~al.(2014)Van Der~Zander, Liskiewicz and
  Textor}]{van2014}
\bibinfo{author}{Van Der~Zander, B.}, \bibinfo{author}{Liskiewicz, M.},
  \bibinfo{author}{Textor, J.}, \bibinfo{year}{2014}.
\newblock \bibinfo{title}{Constructing separators and adjustment sets in
  ancestral graphs}, in: \bibinfo{booktitle}{Conference on Uncertainty in
  Artificial Intelligence}, \bibinfo{publisher}{AUAI Press}. pp.
  \bibinfo{pages}{907--916}.
\bibitem[{Van Der~Zander et~al.(2019)Van Der~Zander, Li{\'s}kiewicz and
  Textor}]{van2019}
\bibinfo{author}{Van Der~Zander, B.}, \bibinfo{author}{Li{\'s}kiewicz, M.},
  \bibinfo{author}{Textor, J.}, \bibinfo{year}{2019}.
\newblock \bibinfo{title}{Separators and adjustment sets in causal graphs:
  Complete criteria and an algorithmic framework}.
\newblock \bibinfo{journal}{Artificial Intelligence} \bibinfo{volume}{270},
  \bibinfo{pages}{1--40}.
\bibitem[{Verma(1993)}]{verma1993}
\bibinfo{author}{Verma, T.S.}, \bibinfo{year}{1993}.
\newblock \bibinfo{title}{Graphical Aspects of Causal Models}.
\newblock \bibinfo{type}{Technical Report} \bibinfo{number}{R-191}. University
  of California, Los Angeles (UCLA), Computer Science Department.
\bibitem[{Wang et~al.(2025)Wang, Tao, Qin and Zhou}]{wang2025}
\bibinfo{author}{Wang, T.Z.}, \bibinfo{author}{Tao, L.}, \bibinfo{author}{Qin,
  T.}, \bibinfo{author}{Zhou, Z.H.}, \bibinfo{year}{2025}.
\newblock \bibinfo{title}{Estimating possible causal effects with latent
  variables via adjustment and novel rule orientation}.
\newblock \bibinfo{journal}{Artificial Intelligence} \bibinfo{volume}{347},
  \bibinfo{pages}{104387}.
\bibitem[{Wang et~al.(2011)Wang, Guo and He}]{wang2011}
\bibinfo{author}{Wang, X.F.}, \bibinfo{author}{Guo, J.H.}, \bibinfo{author}{He,
  X.M.}, \bibinfo{year}{2011}.
\newblock \bibinfo{title}{Finding the minimal set for collapsible graphical
  models}.
\newblock \bibinfo{journal}{Proceedings of the American Mathematical Society}
  \bibinfo{volume}{139}, \bibinfo{pages}{361--373}.
\bibitem[{Wermuth(1987)}]{wermuth1987}
\bibinfo{author}{Wermuth, N.}, \bibinfo{year}{1987}.
\newblock \bibinfo{title}{Parametric collapsibility and the lack of moderating
  effects in contingency tables with a dichotomous response variable}.
\newblock \bibinfo{journal}{Journal of the Royal Statistical Society: Series B
  (Methodological)} \bibinfo{volume}{49}, \bibinfo{pages}{353--364}.
\bibitem[{Xie and Geng(2009)}]{xie2009}
\bibinfo{author}{Xie, X.C.}, \bibinfo{author}{Geng, Z.}, \bibinfo{year}{2009}.
\newblock \bibinfo{title}{Collapsibility for directed acyclic graphs}.
\newblock \bibinfo{journal}{Scandinavian Journal of Statistics}
  \bibinfo{volume}{36}, \bibinfo{pages}{185--203}.
\bibitem[{Xie et~al.(2025a)Xie, Guo and He}]{xie20251}
\bibinfo{author}{Xie, X.D.}, \bibinfo{author}{Guo, J.H.}, \bibinfo{author}{He,
  S.Y.}, \bibinfo{year}{2025}a.
\newblock \bibinfo{title}{Collapsibility of the conditional models of
  {CG}-graphical models}.
\newblock \bibinfo{journal}{Scandinavian Journal of Statistics}
  \bibinfo{volume}{52}, \bibinfo{pages}{1735--1762}.
\bibitem[{Xie et~al.(2025b)Xie, Guo and Sun}]{xie20252}
\bibinfo{author}{Xie, X.D.}, \bibinfo{author}{Guo, J.H.}, \bibinfo{author}{Sun,
  Y.}, \bibinfo{year}{2025}b.
\newblock \bibinfo{title}{Dags as minimal {I}-maps for the induced models of
  causal {B}ayesian networks under conditioning}.
\newblock \bibinfo{journal}{Journal of Machine Learning Research}
  \bibinfo{volume}{26}, \bibinfo{pages}{1--62}.
\bibitem[{Zhang and Zhong(2026)}]{ZHANG2026}
\bibinfo{author}{Zhang, Q.}, \bibinfo{author}{Zhong, H.}, \bibinfo{year}{2026}.
\newblock \bibinfo{title}{A mutual information-driven submodular optimization
  approach to covariate selection in causal effect estimation}.
\newblock \bibinfo{journal}{Knowledge-Based Systems} \bibinfo{volume}{341},
  \bibinfo{pages}{115808}.

\end{thebibliography}



\end{document}